\documentclass{article}

\usepackage{latest_iclr,times}

\usepackage[utf8]{inputenc} 
\usepackage[T1]{fontenc}    
\usepackage{hyperref}       
\usepackage{url}            
\usepackage{booktabs}       
\usepackage{amsfonts}       
\usepackage{nicefrac}       
\usepackage{microtype}      
\usepackage{xcolor}         
\usepackage{algorithm}
\usepackage{algpseudocode}
\usepackage[export]{adjustbox}
\usepackage{caption}
\usepackage{subcaption}
\usepackage{multirow}
\usepackage{enumerate}
\usepackage{makecell}
\usepackage{tikz}
\usepackage{dsfont}
\usepackage{mathtools}
\usepackage{amsmath, amssymb, amsthm}
\usepackage{enumitem}
\usepackage[most]{tcolorbox}
\usepackage{wrapfig}
\usepackage{float}

\usepackage{array}

\definecolor{lavender}{RGB}{230, 230, 249} 
\definecolor{aliceblue}{RGB}{242, 248, 255} 
\definecolor{blu}{RGB}{153, 204, 255} 
\definecolor{purple}{RGB}{149, 115, 239} 
\definecolor{darkpurple}{RGB}{115, 109, 160} 
\definecolor{pink}{RGB}{255, 114, 208}

\newcommand{\DarkPurpleComment}[1]{\hfill{\color{darkpurple}$\triangleright$~#1}}

\newcommand{\E}{\mathbb{E}}

\tcbset{
  mybase/.style={
    coltext=black,
    fonttitle=\bfseries,
    boxrule=0pt,
    arc=2.5mm,
    left=4pt,
    right=4pt,
    top=-4pt,
    bottom=2pt,
    width=0.88\linewidth,
  }
}

\newtcolorbox{optionbox1}[1][]{
  mybase,
  colframe=blu!75!black,
  colback=blu!20!white,
  before=\vspace{1pt},
  after=\vspace{-2pt},
  width=0.65\linewidth,
  #1
}

\newtcolorbox{optionbox2}[1][]{
  mybase,
  colframe=pur!75!black,
  colback=pur!50!white,
  before=\vspace{0pt},
  after=\vspace{0pt},
  width=0.65\linewidth,
  #1
}
\newtcolorbox{optionbox3}[1][]{
  mybase,
  colframe=white,
  colback=purple!6!white,
  before=\vspace{-2pt},
  after=\vspace{-2pt},
  width=\linewidth,
  #1
}

\newtcolorbox{optionbox4}[1][]{
  mybase,
  colframe=orange!75!black,
  colback=orange!50!white,
  before=\vspace{0pt},
  after=\vspace{-2pt},
  width=0.4\linewidth,
  #1
}
\newtheorem{theorem}{Theorem}[section]

\newtheorem{lemma}[theorem]{Lemma}
\newtheorem{assumption}{Assumption}

\newcommand\numberthis{\addtocounter{equation}{1}\tag{\theequation}}

\title{
Principled Top-$k$ Selection for Language\\Models with Hybrid Gradients
}

\author{Xuchen Gong, Junfei Sun, Tian Li \\
University of Chicago\\
\texttt{\{xuchengo,junfeisun,litian\}@uchicago.edu} \\
}

\iclrfinalcopy
\begin{document}

\maketitle
\begin{abstract}
    Selecting the best $k$ items out of $m$ candidates is a critical component of modern large language model systems, such as document selection in Retrieval-Augmented Generation (RAG)  and expert routing in Mixture-of-Experts (MoEs). However, training these selection modules remains challenging due to weak gradient signals and suboptimal exploration-exploitation tradeoffs. Furthermore, prior works often rely on heuristics, lacking principled objectives and approaches that explicitly model and solve the top-$k$ selection problem. In this work, we propose a principled objective for training selection modules, whose gradient naturally provides richer training signals in a hybrid form---containing a supervised-gradient component and a policy-gradient component. We show that the selection problem becomes harder as $m$ increases, and our algorithm converges at rate $O(1/\sqrt{T})$, with the optimal upper bound achieved by balancing between bias and variance. Practically, we apply our method to a set of tasks involving top-$k$ selection, including synthetic regression problems, RAG, and MoE systems, showing that our method outperforms the baselines in next-token prediction perplexity and QA accuracies.
\end{abstract}
\section{Introduction} \label{sec:intro}

In machine learning systems, selecting $k$ items from a total of $m$ candidates is a common and critical component. We consider the following abstraction. Given a training data $(x,y) \in \mathcal{D}$, suppose a model (or part of a model) is composed of a \textit{selection module} $\theta$, which outputs a score of each item $s_{\theta}(i|x)$, and a \textit{prediction module} $\phi$, which takes the selected item $i$ and runs the forward pass to obtain predictions $h_{\phi}(x,i)$. Following the conditional computation graph, if we were able to consider the outputs of all the items, the loss can be formulated as
\begin{align}
 \min_{\theta, \phi} F(\theta,\phi) = \frac{1}{|\mathcal{D}|} \sum_{(x,y)\in\mathcal{D}} \mathcal{L} \left(y, \sum_{i=1}^m p_{\theta}(i|x)  h_{\phi}(x,i)\right), \quad p_{\theta}(i|x) = \frac{e^{s_\theta (i|x)}}{\sum_{j=1}^m e^{s_\theta (j|x)}}.
 \label{eq:true_obj}
\end{align}
Here, $\mathcal{L}$ can be an arbitrary loss such as mean-squared-error in regression and cross-entropy in classification. 
In practice, however, this objective can be impractical and expensive to optimize. The candidate pool size $m$ continuously increases, such as the increasing number of experts in MoE and the increasing size of candidate document pools in RAG systems. As it is infeasible to incorporate all the $m$ candidates during training or inference, we typically select the most promising $k$ items to generate predictions by training top-$k$ selection modules, such as routers in MoE and rerankers in RAG. A common compromise is to simply apply the summation over (noisy) top-$k$ items \citep{in-context,lewis2020retrieval,realm}, i.e., selecting  $I_x = \arg\text{top-}k_{i\in[m]}(s_{\theta}(i|x))$, and optimizing 
\looseness=-1
{\setlength{\abovedisplayskip}{4pt}
\setlength{\belowdisplayskip}{4pt}
\begin{align}
 \min_{\theta, \phi} \frac{1}{|\mathcal{D}|} \sum_{(x,y)\in\mathcal{D}} \mathcal{L} \left(y, \sum_{i\in I_x} p_{\theta}(i|x)  h_{\phi}(x,i)\right), \quad p_{\theta}(i|x) = \frac{e^{s_\theta (i|x)}}{\sum_{j\in I_x} e^{s_\theta (j|x)}}.
 \label{eq:existing_obj}
\end{align}
}

However, several challenges in top-$k$ selection still remain. There exist numerous updating rules that do not explicitly model the selection procedure into the optimization problem, but rather rely on generic end-to-end training losses to update all trainable variables~\citep{outrageously,deepseekmoe,switchtransformer}. Understanding the dynamics of the underlying selection mechanisms remain open. 
Even if the selection objective exists, the connection between existing formulations (such as Eq. \eqref{eq:existing_obj}) and the true objective Eq. \eqref{eq:true_obj} is unclear.  {In addition}, the top-$k$ operation is generally not differentiable, and the gradient signal for training the selection module can be suboptimal. For instance, in Eq.~\eqref{eq:existing_obj}, the gradient of the selection module $\theta$ flows via the aggregation weights in the truncated summation over $k$ items. Therefore, the gradient over $\theta$ teaches how to adjust weights of the \textit{already selected} items in $I_x$, but does not directly supervise \textit{what items to select among all $m$ options}. 
This lack of training signals for selection has been a well-known bottleneck of training selection modules \citep{liu2023sparse,panda2025dense}, and the issue is more pronounced as $m$ increases, leading to a growing search space $O(m^k)$. \looseness=-1

When the loss $\mathcal{L}$ is the negative log-likelihood,  
the per-sample objective can equivalently be expressed in terms of maximizing $p(y|x)$. Introducing either an item-level latent variable $i$ or a subset-level latent variable $I$, we can decompose
$p(y|x)$ as $ p(y|x) = \sum_{i} p(y|x,i) p(i|x) = \sum_{I} p( y|x,I) p(I|x)$. This partially motivates our proposed set-level objective. Additionally, unbiased sampling following an item-level distribution may result in selecting the same item more than once, which does not match common practical implementations. 
Therefore, we propose the following surrogate objective where we perform top-$k$  selection at the \textbf{set level}, as opposed to the item level: 
{\setlength{\abovedisplayskip}{5pt}
\setlength{\belowdisplayskip}{5pt}
\begin{align}
\label{eq:our_obj}
    J(\theta,\phi)  =  \frac{1}{|\mathcal{D}|}\sum_{(x,y)\in\mathcal{D}} J_x(\theta,\phi) =  \frac{1}{|\mathcal{D}|}\sum_{(x,y)\in\mathcal{D}} \mathbb{E}_{I \sim \pi_{\theta}(\cdot|x)} \left[ \mathcal{L}(y, f_{\theta,\phi}(x, I)) \right],\quad |I|=k,
\vspace{-1em}
\end{align}
\noindent where $J_x$ denotes the per-sample loss over $(x,y)$,
}
$\pi_{\theta}(\cdot|x)$ is a distribution \textit{over subsets with cardinality $k$}, and $f_{\theta,\phi}(x, I)$ denotes a general weighted aggregation function to combine selected items in $I$. In Section~\ref{sec:problem:obj}), we discuss some specific requirements on $\mathcal{L}$ and $f_{\theta,\phi}(x, I)$ such that  Objective~\eqref{eq:our_obj} can be viewed as an upper bound of Objective~\eqref{eq:true_obj}. 

\begin{wrapfigure}{r}{0.25\textwidth}
  \centering     
  \vspace{-0.5cm}
  \includegraphics[width=\linewidth]{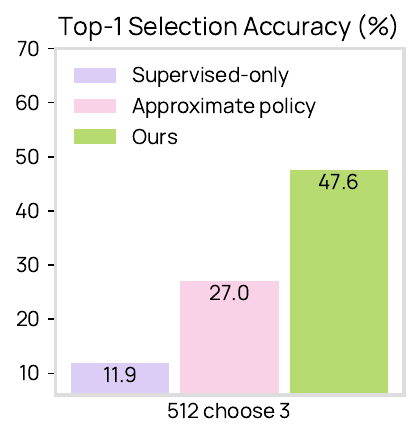}
  \caption{Our method outperforms prior ones that abandon or approximate the policy gradient. 
  } 
  \label{fig:toy_512c3}
  \vspace{-0.5cm}
\end{wrapfigure}
We develop an algorithm to optimize this objective. The unbiased gradient of  Objective~\eqref{eq:our_obj} naturally comprises two hybrid terms that learn both selection and weight adjustments over the selected subset. As we detail in Section~\ref{sec:problem:obj}, the first term corresponds to a policy gradient, and the second corresponds to a classic  supervised gradient. As the policy gradient could have high variance, we additionally introduce a scalar $\lambda$ in front of the policy gradient to adjust the  tradeoff between the two terms. In Figure \ref{fig:toy_512c3}, our method achieves an improved bias/variance tradeoff and outperforms existing approaches with only the supervised gradient   \citep{supervisedonly} or approximating the policy gradient \citep{liu2023sparse,liu2024grin}  by more than 20\% in top-1 selection accuracy (details in Section \ref{sec:regression}). 
\looseness=-1

Theoretically, we prove that our training method converges to  stationary points of our objective and to neighborhoods of stationary points of the true objective \eqref{eq:true_obj} at a rate of $O(1/\sqrt{T})$. The bias/variance tradeoff and thus the convergence upper bound is adjusted by $\lambda$, whose optimal value is problem-dependent. 
We explore a synthetic regression problem to illustrate the difficulty of learning the selection module and the superiority of our method in exploring more subsets and improving the final performance. 
Furthermore, we apply our method to language models, including finetuning rerankers for in-context RAG and training MoE models from scratch. We observe that our method consistently outperforms the existing methods under all candidate pool sizes $m$. Our contributions are summarized as follows. 

\textbf{Contributions.}
    We formalize the problem of sparse top-$k$ selection in language models and propose a surrogate objective that directly optimizes over selected subsets of items, which also serves as an upper bound of the true objective for certain functions. This objective naturally yields hybrid gradients with two components (policy and supervised parts) that offer richer gradient signals for  top-$k$ selection than prior approaches. We establish convergence guarantees of our algorithm, and analyze the  optimal bias/variance tradeoff.
     We evaluate our approach on different losses and tasks, including mean-squared error regression and cross-entropy loss on language tasks. Particularly, we fine-tune rerankers in RAG systems for language modeling and question answering (QA), and we train MoE models for different numbers of experts for language modeling from scratch. Our approach consistently outperforms strong baselines in terms of perplexity and task-specific metrics. 


\section{Related Work} \label{sec:related}
\vspace{-0.5em}
\textbf{Top-$k$ selection.} Selecting a sparse subset of $k$ items from a large candidate pool is a challenging problem due to the non-differentiable top-$k$ operator. A line of work proposes differentiable ranking and relaxations of discrete selection, such as stochastic softmax tricks \citep{paulus2020gradient}, optimal-transport-based approximations \citep{xie2020differentiable}, PiRank \citep{swezey2021pirank}, Laplace-based relaxations \citep{struski2025lapsum}, and linear-time top-$k$ operators for recommendation \citep{zhu2025differentiable}. While these methods improve gradient-based optimization, they typically require dense forward passes over all candidates or violate strict sparsity constraints \citep{rosenbaum2017routing,bengio2015conditional}, making them incompatible with applications considered in this work such as strict $k$-expert MoE routing. 
\looseness=-1

\textbf{Retrieval-augmented generative models.} Classical RAG training optimizes a marginal likelihood over latent documents, typically approximated using deterministic top-$k$ selection \citep{izacard2023atlas,lewis2020retrieval,guu2020retrieval,borgeaud2022improving,izacard2021leveraging}. More recent approaches improve retrieval via reranking \citep{nogueira2019multi,nogueira2019passage,rankt5,qin2024large} or alignment with language model signals \citep{replug,in-context}. Some work optimizes retrieval with reinforcement learning signals such as answer correctness or heuristic feedback \citep{li2025r3,ma2023query,zhang2024reinforcement,basu2024statistical}. In contrast, our method yields an unbiased hybrid gradient in which the policy gradient component corresponds to a single-sample REINFORCE estimator \citep{williams1992simple,sutton1999policy}.
\looseness=-1

\textbf{Mixture-of-experts.}
MoE architectures scale model capacity via sparse activation of experts \citep{jacobs1991adaptive,jordan1994hierarchical,outrageously,switchtransformer,lepikhin2020gshard,lewis2021base,zhou2022mixture,zoph2022st,deepseekmoe,mixtralmoe}, but training the router remains challenging due to suboptimal gradient signals flowing through discrete top-$k$ selection. Prior work relies on load-balancing losses \citep{riquelme2021scaling,lepikhin2020gshard,outrageously,switchtransformer,deepseekmoe,lewis2021base} or noisy gating \citep{outrageously,supervisedonly,switchtransformer,lepikhin2020gshard,zoph2022st,riquelme2021scaling} for improved exploration/exploitation tradeoffs in expert selection. More recent work improves gradient flow through the router, including proxy gradients \citep{panda2025dense} and higher-order approximations to the policy gradient \citep{liu2023sparse,liu2024grin}. However, many of these approaches remain heuristics, lack gradient signals for unselected items,  or follow objectives that are inconsistent with practical deployment by sampling each expert independently. In contrast, our framework explicitly models subset selection as stochastic sampling over unordered $k$-sized sets and derives a principled objective and unbiased hybrid gradient composed of policy and supervised parts. 
\looseness=-1

\section{Objective and Hybrid Gradients} \label{sec:problem}

Let $\mathcal{D}$ be the training set containing a collection of samples $\{(x, y)\}$ and $\mathcal{L}(y,\hat y)$ be a general per-sample loss, such as mean squared error in regression or negative log-likelihood in language modeling. 
In various language models involving selecting $k$ items, the optimization variables consist of two groups: the selection parameters $\theta$ (e.g.,  rerankers in RAG or the router in MoE), and the remaining prediction parameters $\phi$ (e.g., the language model in RAG or the experts in MoE). The selection module $\theta$ produces a score $s_{\theta}(i|x), i\in[m]$ for each  item, from which we obtain a \textit{probability distribution over subsets} $\pi_\theta(I|x)$ where $|I|=k$. We then sample a subset $I$ from $\pi_\theta(I|x)$ and feed it to the weighted aggregation function $f_{\theta,\phi}(x,I)$. 
In Figure \ref{fig:framework}, we illustrate the forward pass of one selection decision and prediction (left) and its instantiations in RAG (middle) and MoE (right). 


\begin{figure}[h!]
    \centering
    \includegraphics[width=\linewidth]{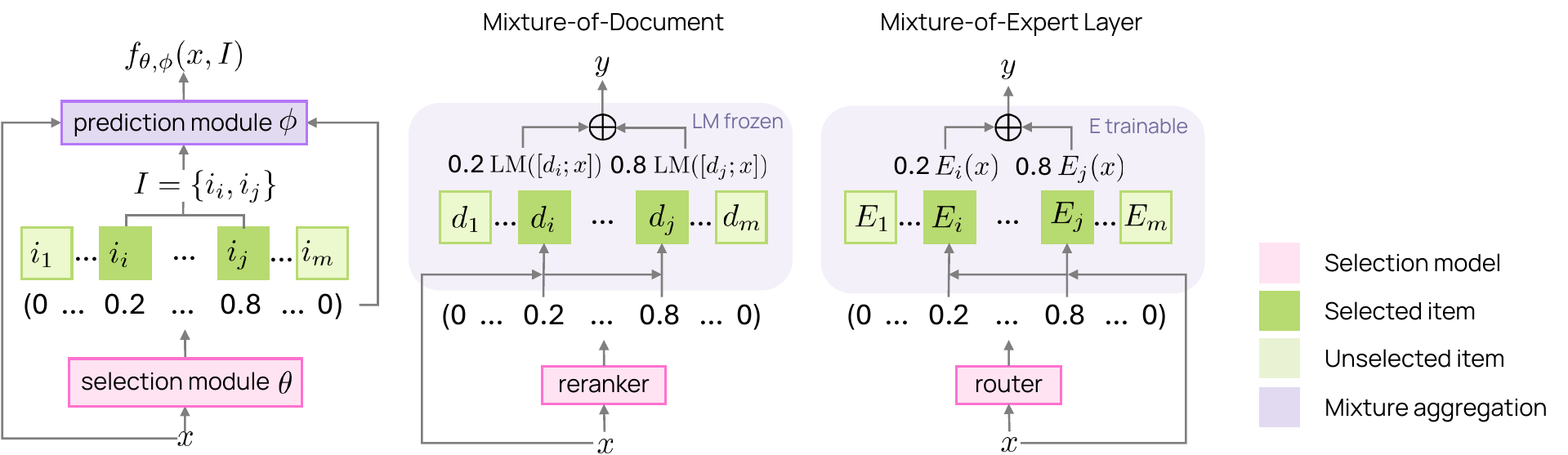}
    \caption{Our framework (left) and its applications to RAG and MoE. Middle: In RAG, given input query $x$, the reranker selects a subset of the most relevant documents to augment $x$ to predict the target $y$. RAG has one selection and prediction module in total, with the prediction module (i.e., LM) frozen in our setting. Right: In each MoE layer, given the input hidden state $x$, the router selects a subset of expert models to compute the output hidden state $y$. The full MoE model contains multiple MoE layers of selection and aggregation, with both the selection and prediction modules trainable. }
    \vspace{-1em}
    \label{fig:framework}
\end{figure}

In the following, we describe the general form of our proposed surrogate objective and its gradients in Section \ref{sec:problem:obj}, discuss methods of sampling from  $\pi_{\theta}$ in Section \ref{sec:problem:selection}, and show how the prediction module performs the subsequent forward computation Section \ref{sec:problem:prediction}. 

\vspace{-0.5em}
\subsection{Surrogate Objective and Hybrid Gradients} \label{sec:problem:obj}
\vspace{-0.5em}

In Eq.~\eqref{eq:our_obj}, we present the objective assuming the number of selection modules is one. In this section, we first examine this objective in detail, discuss its gradients, and then extend it to the case of having multiple selection modules in the model architecture.

Recall that $ p(y|x) = \sum_{i=1}^m p_{\theta}(i|x) p(y|x,i) = \sum_{I} p(I|x) p( y|x,I) $. When the output $ f_{\theta,\phi}(x,I)$ defines a consistent conditional probability given $I$ such that $ f_{\theta,\phi}(x, I)=p(y|x, I)$ 
and loss $\mathcal{L}$ is convex, such as negative log-likelihood in language modeling, our objective (Eq.~\eqref{eq:our_obj}) can be viewed as an upper-bound of the true objective (Eq. \eqref{eq:true_obj}). Specifically, we apply $\sum_{i=1}^m p_{\theta}(i|x) h_{\phi}(x,i) = \sum_{I} \pi_{\theta}(I|x) f_{\theta,\phi}(x,I)$ to Eq. \eqref{eq:true_obj} and by Jensen's inequality, for each sample $x$, we have 
\begin{align}
    \mathcal{L} \left(y, \sum_{I} f_{\theta,\phi}(x,I) \cdot \pi_{\theta}(I|x) \right) \leq \sum_{I} \mathcal{L} (y, f_{\theta,\phi}(x,I)) \cdot \pi_{\theta}(I|x) = \mathbb{E}_{I\sim\pi_{\theta}} [\mathcal{L} (y, x,I;\theta,\phi) ],
    \label{eq:jensen}
\end{align}
where we slightly switch notations in the last step, using $\mathcal{L}(y, x, I; \theta, \phi)$ and $\mathcal{L}(y, f_{\theta,\phi} (x,I))$ interchangeably. 
Since the per-sample loss $J_x(\theta, \phi)=\mathbb{E}_{I \sim \pi_{\theta}(\cdot|x)} \left[ \mathcal{L}(y, x, I; \theta, \phi)\right]$ and  the prediction module $\phi$ is independent of $\theta$, we have 
\begin{align*}
 \nabla_{\phi} J_x(\theta,\phi)  = \sum_{I} \pi_{\theta}(I|x) \nabla_{\phi} \mathcal{L}(y, x, I;\theta,\phi) = \mathbb{E}_{I \sim \pi_{\theta}} \left[ \nabla_{\phi} \mathcal{L}(y, x, I;\theta,\phi)\right].
\end{align*}
In contrast, the selection module $\theta$ affects the per-sample loss $J_x(\theta, \phi)$ in two ways: (i) via subset sampling $\pi_{\theta}(I|x)$ and (ii) via aggregation $f_{\theta,\phi}(x,I)$. We define leave-one-out batch loss $b_x$ as
\looseness=-1
\begin{align}
    b_x := \frac{1}{|\mathcal{B}|-1} \sum_{(x', y')\in \mathcal{B},~ (x', y')\neq (x,y)} \mathcal{L}(y',x',I;\theta,\phi) \quad  \text{for each }x.\label{eq:baseline}
\end{align}
Then applying the product rule and the log-derivative trick yields

\begin{optionbox3}
\begin{align*}
    \nabla_{\theta} J_x(\theta,\phi)
    &= \mathbb{E}_{I \sim \pi_{\theta}} \left[ \underbrace{(\mathcal{L}(y, x, I;\theta,\phi)-b_x) \cdot \nabla_{\theta} \log \pi_{\theta}(I|x)}_{\text{policy gradient}} + \underbrace{\nabla_{\theta} \mathcal{L}(y, x, I;\theta,\phi)}_{\text{supervised gradient}} \right]. \numberthis \label{eq:grad}
\end{align*}
\end{optionbox3}

In the above, we use the leave-one-out batch loss as the baseline $b_x$, and the full derivation and unbiasedness of Eq. \eqref{eq:grad} is proved in Appendix \ref{sec:app:method}. The gradient is composed of two components, the policy gradient and the supervised gradient. While such decomposition is standard for stochastic computation graphs~\citep{kool2020estimating} (named score functions and pathwise term in a different context), in our setting the two terms play complementary roles. The policy gradient provides training signals for \textit{`what items to select'} and learns to explore unselected items. The supervised gradient, in contrast, provides signals for \textit{`how to weigh each item given the selected set'} and improves exploitation.  
\looseness=-1

Although Eq. (\ref{eq:grad}) is the unbiased gradient, the policy gradient is usually missing \citep{dselect,outrageously} or approximated \citep{liu2023sparse,liu2024grin} in existing literature due to the high variance of the raw policy gradient \citep{clark2022unified,rosenbaum2019routing}. In this work, we introduce a scalar $\lambda$ in front of the policy gradient to adjust the bias/variance tradeoff: As $\lambda \to 0$, our gradient is biased with respect to the objective but of low variance; as $\lambda\to 1$, our gradient is unbiased but its variance increases. In our experiments, we conduct a small grid search over  $\lambda$ and show that the optimal test performance is achieved at $\lambda\in (0,1)$.
\looseness=-1

\textbf{Implementation.} Our approach is easy to implement in practice. Denoting \text{SG}($\cdot$) as the stop-gradient operator, we can perform backpropagation on the following per-sample loss (given input $(x, y)$) to obtain the hybrid gradients with hyperparameter $\lambda$:
\begin{align*}
    \texttt{loss} \coloneqq \lambda \;\text{SG}(\ell_x - b_x) \cdot \log \pi_{\theta}(I|x)  + \ell_x, \quad \text{where} \;\; \ell_x = \mathcal{L}(y,x,I;\theta,\phi), \; I\sim\pi_{\theta}(\cdot|x).
\end{align*}

\textbf{Extending to $L$ layers.} We denote $L$ as the number of selection and prediction modules (or layers) in the model architecture. The general form of extending of our proposed Objective~\eqref{eq:our_obj} to $L$ selection modules is as follows. We propose to find the selection parameters $\theta^{(1:L)}$ and prediction parameters $\phi^{(1:L)}$ such that the average loss under stochastic choices of the subsets $I^{(1:L)}$ is minimized:
{\setlength{\abovedisplayskip}{4pt}
\setlength{\belowdisplayskip}{4pt}
\begin{align}
\label{eq:our_obj_multilayer}
    \min_{} \frac{1}{|\mathcal{D}|}\sum_{(x,y)\in\mathcal{D}} \mathbb{E}_{I^{(1:L)} \sim \pi_{\theta^{(1:L)}}(\cdot|x)} \left[ \mathcal{L}(y, x, I^{(1:L)};\theta^{(1:L)},\phi^{(1:L)}) \right].
\end{align}
}
\par\noindent The corresponding gradients and implementation for $L>1$ are similar to the case of $L=1$, and we present the exact format in Section \ref{sec:moe}. 

\vspace{-0.5em}
\subsection{Sampling $I \sim \pi_{\theta}(\cdot | x)$} \label{sec:problem:selection} 
\vspace{-0.5em}

We now describe how we define and sample from the set-level distribution
$\pi_\theta(\cdot | x)$ over subsets $I \subseteq [m]$ of size $k$. More generally, such subset distributions have been studied previously in the context of sequential sampling without replacement~\citep{plackett1975analysis,kool2020estimating}. Given item scores
$s_\theta(i|x)$, we first define a categorical distribution over individual items with temperature $\tau$:
$p_{\theta}(i|x) = \frac{\exp(s_{\theta}(i|x)/\tau)}{\sum_{j=1}^m \exp(s_{\theta}(j|x)/\tau)}$.
 To obtain $k$ distinct items, we sample sequentially from
$p_\theta(\cdot | x)$ without replacement, obtaining an ordered list $L=\{L_1,\ldots,L_k\}$. The probability of sampling this particular order is $
    p_{\theta}(L_1, \ldots, L_k | x) = \prod_{t=1}^k \left[ \frac{ p_{\theta}(L_t|x) }{ 1 - \sum_{s<t} p_{\theta}(L_s|x) } \right]$.
Since our applications use the selected items as an unordered set, we define the probability of obtaining an unordered set $I$ by summing the
probabilities of all possible orderings of its elements.  
Let $\mathcal{S}$ denote the set of all permutations of $L$. Then we have 
{\setlength{\abovedisplayskip}{4pt}
\setlength{\belowdisplayskip}{4pt}
\begin{align}
 \pi_{\theta}(I | x) = \sum_{L'\in \mathcal{S}} p_{\theta}(L'_1, \ldots, L'_k | x). \label{eq:unordered}
\end{align}}
Lemma \ref{lemma:logpi} (proof in  Appendix \ref{sec:app:method}) below formalizes this construction.

\begin{lemma}[\cite{kool2020estimating}]\label{lemma:logpi}
    Sampling $I$ from $\pi_{\theta}$ in Eq. (\ref{eq:unordered}) can be implemented by \texttt{L=multinomial($p_{\theta}(\cdot|x)$, $k$, replacement=False)} and then using $L$ as a set.
\end{lemma}




\subsection{Application to Language Models} \label{sec:problem:prediction}
\vspace{-0.5em}

After the selection module chooses a subset of items $I$, the prediction module combines the outputs associated with the selected items. The form of this aggregation depends on the application.
In our framework, RAG corresponds to a single selection stage with $L=1$. A reranker samples a set of documents following Eq.~\eqref{eq:unordered}, and a frozen language model computes a prediction conditioned on each selected document. These document-conditioned predictions are then aggregated based on the probabilities of each document (details in Section~\ref{sec:rag}). In contrast, an MoE model contains multiple selection stages with $L>1$. At each MoE layer, a router selects a subset of experts, and their outputs are combined to form the hidden state passed to the next layer. We discuss detailed MoE procedures in Section~\ref{sec:moe}. Thus, RAG performs output-level aggregation once, whereas MoE performs hidden-state aggregation for $L$ times throughout the network.

In the following, we discuss the proposed training procedure in Algorithm \ref{alg:MoT} for $L=1$, and we present the natural extension for $L>1$ in Section \ref{sec:moe}. In line 7, $\mathcal{L}$ represents the MSE loss for regression tasks (Section \ref{sec:regression}) or the cross-entropy loss for RAG and MoE applications (Sections \ref{sec:rag}$-$\ref{sec:moe}).

{\setlength{\textfloatsep}{6pt}
\setlength{\intextsep}{6pt}
\begin{algorithm}
\caption{Training the Selection and Prediction Module with Hybrid Gradients}
\label{alg:MoT}
\begin{algorithmic}[1]
\State \textbf{Input:} Selection parameters $\theta$, prediction parameters $\phi$, $\eta$, $\lambda$, $\tau$
\For{each batch $\mathcal{B}$}
    \For{$(x,y)\in \mathcal{B}$}
        \State $p_{\theta}(\cdot|x) \gets \text{softmax}(s_{\theta}(\cdot|x)/ \tau)$
        \State Sample $k$ items $I \sim \text{Multinomial}(p_{\theta}(\cdot|x))$ without replacement \DarkPurpleComment{Selection (\S \ref{sec:problem:selection})}
        
        \State Compute the probability of the set $\pi_{\theta}(I|x)$ using Eq. (\ref{eq:unordered})
        \State $\ell_{x} \gets \mathcal{L}(y,x,I;\theta,\phi)$ \DarkPurpleComment{Aggregation and compute loss (\S \ref{sec:regression}$-$\ref{sec:moe})}
    \EndFor
    \State Compute $b_x \gets$ leave-one-out loss average for each $x\in\mathcal{B}$ by Eq. \eqref{eq:baseline}
    \State $\ell_\mathcal{B}
    \gets \frac{\lambda}{|\mathcal{B}|} \sum_{x\in \mathcal{B}}\text{SG}\left(\ell_x - b_x\right) \cdot \log \pi_{\theta}(I|x) + \frac{1}{|\mathcal{B}|} \sum_{x\in \mathcal{B}} \ell_{x}$ 
    \State $[\theta,\phi] \gets [\theta,\phi] - \eta\nabla_{[\theta,\phi]} \ell_\mathcal{B}$
\EndFor
\end{algorithmic}
\end{algorithm}
}
\section{Convergence Guarantees} 
\label{sec:theory}
\vspace{-0.5em}

In this section, we prove the convergence of our method to our surrogate objective $J$ (Eq. \eqref{eq:our_obj}) and the true objective $F$ (Eq. \eqref{eq:true_obj}) at rate $O(1/\sqrt{T})$, with the optimal upper-bound obtained by adjusting $\lambda$ for bias/variance tradeoff. We consider optimization over all the parameters $w:=(\theta,\phi)$ with $L=1$. The selection module produces logits $z=s_\theta(x)$ and the prediction module outputs $h_{\phi}(x,i), i\in[m]$, which we abbreviate as $h_{i,\phi}(x)$. In the following, we present the assumptions and main theorems, with the proof detailed in Appendix \ref{sec:app:theory}. 

\begin{assumption}[Compact domain]
\label{ass:compact-domain}
There exists a compact set $\Omega$ s.t. $w_t=(\theta_t,\phi_t)\in \Omega, \forall t\ge 0$, in which the selection module is discriminative, i.e., the minimal logit difference between the selected and unselected item along the optimization trajectory is positive: $$s_{(k)}(x;\theta_t)-s_{(k+1)}(x;\theta_t)\geq \gamma_{\min}>0, \forall x.$$
\end{assumption}

\begin{assumption}[Bounded input]
\label{ass:bounded-experts-inputs} 
The input samples have compact support: $\|x\| \leq C_x, \forall x\in \mathcal{D}$.
\end{assumption}

\begin{assumption}[Loss regularity]
\label{ass:loss-smooth}
$\mathcal{L}(y,\hat y)$ is $\bar L$-Lipschitz and $\bar\beta$-smooth in $\hat y$.
\end{assumption}

\begin{assumption}[Model smoothness]
\label{ass:model-smoothness} $s_{\theta}$ and $h_{i,\phi}, i\in[m]$ are smooth in $\theta,\phi$ and their inputs.
\end{assumption}

We remark that standard transformer blocks with smooth activations~\citep[e.g.,][]{gelu, silu, swish} satisfy Assumption \ref{ass:model-smoothness}. In addition, under Assumptions \ref{ass:compact-domain}$-$\ref{ass:model-smoothness}, we prove in Appendix \ref{sec:app:theory} that the following constants are bounded: $\sigma_S^2$ (stochastic supervised gradient variance bound), $\sigma_P^2$ (stochastic policy gradient variance bound), $\epsilon$ (upper-bound of the full-batch gradient bias $\|\nabla_w J - \nabla_w F\|$ from objective truncation), and $R_{\max}$ (the maximal full-batch policy gradient norm on the optimization trajectory).

\begin{theorem}[Convergence]
\label{thm:convergence}
Recall that $J$ denotes our surrogate objective (Eq. \eqref{eq:our_obj}) and $F$  denotes the true objective (Eq. \eqref{eq:true_obj}). Under Assumptions~\ref{ass:compact-domain}-\ref{ass:model-smoothness}, let $\eta = 1/(\beta\sqrt{T})$ and $T\geq 4$.
For any fixed $\lambda\in[0,1]$, Algorithm~1 satisfies
{\setlength{\abovedisplayskip}{4pt}
    \setlength{\belowdisplayskip}{4pt}
\begin{align}
\frac{1}{T}\sum_{t=0}^{T-1}\E\|\nabla J(w_t)\|^2
&\leq
\frac{2\beta(J(w_0)-J^*) + 2\sigma_S^2 + 2\lambda^2\sigma_P^2}{\sqrt{T}}
+3(1-\lambda)^2R_{\max}^2,
\label{eq:conv-J}\\
\frac{1}{T}\sum_{t=0}^{T-1}\E\|\nabla F(w_t)\|^2
&\leq
\frac{2\beta(F(w_0)-F^*) + 2\sigma_S^2 + 2\lambda^2\sigma_P^2}{\sqrt{T}}
+3\bigl(\epsilon+(1-\lambda)R_{\max}\bigr)^2.
\label{eq:conv-F}
\end{align}}
\end{theorem}

Full proofs are given in Appendix~\ref{sec:app:theory}. $\lambda=1$ corresponds to the unbiased gradient (Eq. (\ref{eq:grad})), which gives convergence to the stationary point of $J$. The hybrid gradient  provides a smooth bias-variance tradeoff controlled by $\lambda$: On the one hand, increasing $\lambda$ reduces the bias and therefore lowers the asymptotic bias floor in the stationarity bound. On the other hand, increasing $\lambda$ increases the variance of the policy gradient. In practice, the policy gradient variance can be very large, so we empirically search in a grid of three to four values in our experiments, observing that optimal $\lambda$'s tend to be small. 
\vspace{-0.5em}
\section{Regression on Synthetic Data} \label{sec:regression}

To rigorously study the properties of the top-$k$ selection problem, we design a stylized synthetic regression problem with ground-truth labels indicating which items should be selected, and instantiate $\mathcal{L}$ as a mean squared error (MSE) loss. This allows us to measure selection accuracy and diagnose the selection mechanism. 
Assume there are $m$ candidate items in total, each represented by a fixed vector $v_i\in\mathbb{R}^h$ for $i\in[m]$.  We generate the ground-truth selection parameters $\theta^{\star}\in\mathbb{R}^{m\times h}$, all the items $\{v_i\}_{i \in [m]}$, and each input feature $x\in \mathbb{R}^h$ from Gaussian distributions. For each $x$, we obtain the ground-truth subset indices $I_x=\arg\text{top-}k(\theta^{\star}x)$. We then generate optimal weights $p=\text{softmax}(\text{mask-topk}(\theta^{\star} x)) \in \mathbb{R}^k$ where $\text{mask-topk}(a)= a$ if $a\in\text{top-}k$ and $ -\infty$ otherwise. Finally, the label $y$ for input $x$ is generated via a linear combination of these $k$ items, i.e., $y =\sum_{i\in I_x} p_i  v_i \in\mathbb{R}^h$. 
Given $\{v_i\}_{i \in [m]}$ and $\mathcal{D}=\{(x,y)\}$, we aim to learn the selection parameters $\theta$ so as to minimize the empirical MSE loss $\frac{1}{|\mathcal{D}|} \sum_{(x,y)\in\mathcal{D}}\|y - \sum_{i\in  I_{x}}  p_i v_i \|^2$, $|I_x|=k$.

\textbf{Data, models, and baselines.} We take $h=128, k=3$, and $m\in\{64,128,256,512\}$. We generate 10K training samples, 1K validation samples, and 1K test samples. We use the linear selection model $s_\theta: x\to \theta x$ and aggregate the selected hidden states via
$
f_{\theta,\phi}(x, I) = \sum_{i\in I_x}  \frac{e^{s_\theta (i|x)}}{\sum_{j\in I_x} e^{s_\theta (j|x)}}  v_i.
$
We compare with five classic baselines in the top-$k$ selection literature, which are also common methods in MoE work. Details of methods and hyperparameter tuning are in Appendix \ref{sec:app:toy}. The baselines are: (1) Deterministic top-$k$ which selects the $k$ items based on the raw scores,  (2) Gaussian top-$k$ \citep{outrageously}, which adds Gaussian noise to item scores with learnt noise variance, (3) DefaultMoE \citep{panda2025dense}, which uses each expert's running average of outputs for aggregation when the experts are not selected, (4) SparseMixer \citep{liu2023sparse}, which uses high-order derivatives to approximate the policy gradient, and (5) Supervised-only, which is implemented by the classic Gumbel-Top-$k$ trick \citep{kool2019stochastic} and corresponds to our method with $\lambda=0$.

We report the MSE loss and the selection accuracies on this toy problem, as we have access to the ground-truth selected items. In Figure \ref{fig:opt_difficulty},  we present the average performance of the five selection baselines under different $m$'s. We observe that as $m$ increases, the optimization becomes more challenging: The test top-1 selection accuracy drops drastically, and test MSE loss increases. Furthermore, in Figure \ref{fig:toy}, we show that our method outperforms all baselines in terms of top-1 selection accuracy and MSE loss under the optimal $\lambda$. The exact numerical values are reported in Table  \ref{tab:regression_result}, and an ablation on $\lambda$ is presented in Figure \ref{fig:regression_ablation} in the appendix. 


\begin{figure}[!h]
\centering
  \begin{minipage}[t]{0.35\textwidth}
    \centering
    \includegraphics[width=0.86\textwidth]{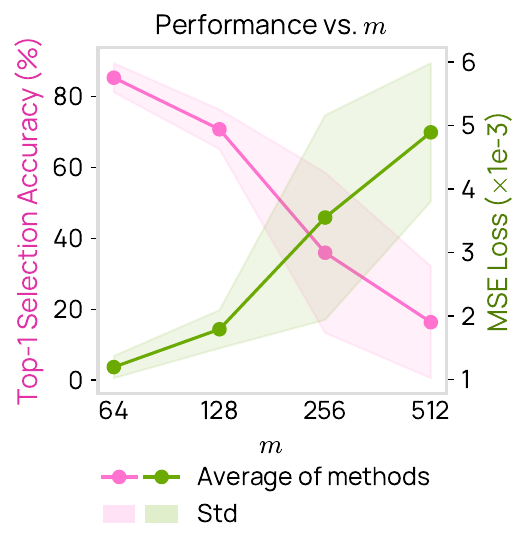}
    \vspace{-0.5em}
    \caption{Selection module is harder to optimize as $m$ increases.}
    \label{fig:opt_difficulty}
  \end{minipage}
  \hfill
  \begin{minipage}[t]{0.6\textwidth}
    \centering
    \includegraphics[width=0.83\textwidth]{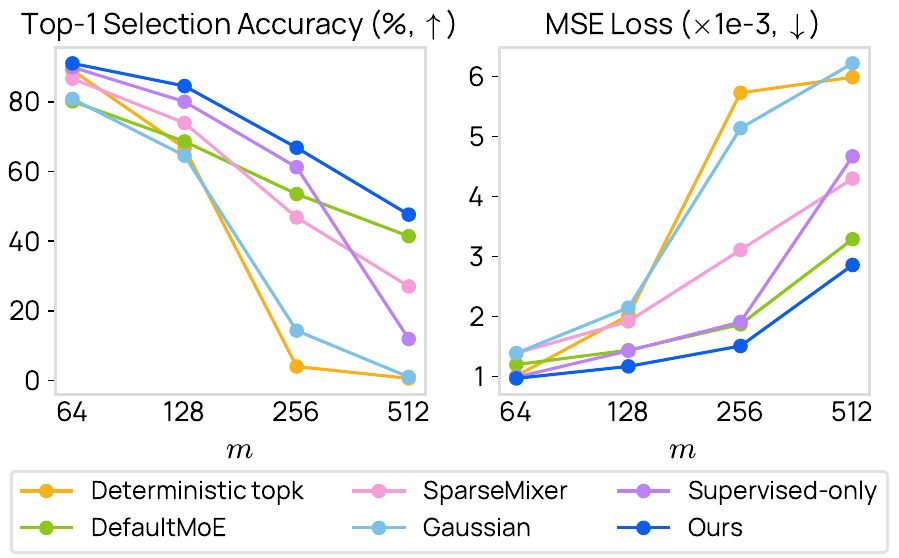}
    \vspace{-0.5em}
    \caption{Our method outperforms baselines on the regression problem in test top-1 selection accuracy and MSE loss.}
    \label{fig:toy}
  \end{minipage}
\vspace{-0.5em}
\end{figure}

We further visualize the evolution of $\theta$'s selection decision (on two test samples) during optimization with different methods (Figure \ref{fig:1D_landscape}). Our method traverses more diverse selection decisions than the deterministic top-$k$ and supervised-only baselines, which illustrates the benefits of the policy gradient.


\vspace{-0.5em}
\section{Document Selection for Retrieval-Augmented Generation} \label{sec:rag}
\vspace{-0.5em}

In this section, we apply our method to in-context RAG, which generally involves three model components: a retriever that selects $m$ documents from a datastore of size $M$, a reranker that selects $k$ documents out of $m$, and a language model that performs next-token (or next-sequence) prediction based on the concatenation of the original query and the selected document(s). We consider the common setting of training the reranker $\theta$ only \citep{replug,in-context}.

Denote $x$ as the query sequence and $y$ as the ground-truth continuation of $x$. Given $(x,y)$, the reranker produces a score $s_{\theta}(i|x)$ for each document $i\in[m]$, indicating the relevance of each document to the query $x$. The reranker next chooses a set of documents $I=\{I_1, \ldots, I_k\}$ via sampling following Eq. (\ref{eq:unordered}). For each selected document $I_i$, we concatenate it with query $x$ to form $[I_i;x]$ and compute $p_{\text{LM}}(y|x,I_i)$, the probability of the target continuation given input $[I_i;x]$. Then the output $f_{\theta,\phi}(x,I)$ is the probability of the continuation $y$ computed by 
{\setlength{\abovedisplayskip}{2pt}
\setlength{\belowdisplayskip}{2pt}
\begin{align}
p(y|x,I) = \sum_{i=1}^k a_i(x,I) \cdot p_{\text{LM}}(y|x,I_i), \quad a_i(x,I) = \frac{\sum_{\substack{L:\{L_1,\ldots,L_k\}=I \\ L_1=i}} \pi_{\theta}(L|x)}{\pi_{\theta}(I|x)}.
\label{eq:inequality}
\end{align}}
The proof that our choice of $a_i(x,I)$ in Eq. \eqref{eq:inequality} makes our objective upper-bound the true objective is in Appendix \ref{sec:app:rag}. In our experiments, we follow prior work \citep{replug,in-context} to take both $x$ and $y$ as a sequence of $s$ tokens for retrieval efficiency. In language modeling, we use $s=128$, so the task is to predict the next-chunk based on the previous chunk, consistent with \citet{replug}. In question answering, $x$ is the question and $y$ is the answer. 

We fine-tune the reranker MS Marco\footnote{https://huggingface.co/cross-encoder/ms-marco-MiniLM-L12-v2}, whose selected documents are concatenated with the original query and fed as input to generative models, including Qwen-2.5-1.5B \citep{qwen2.5} and Llama-3.2-3B \citep{llama3}. We evaluate our method and the baselines on Wikitext-103 \citep{wikitext103} and open-domain QA datasets, including Natural Questions (NQ) \citep{nq} and TriviaQA \citep{triviaqa}. Following prior work \citep{realm,replug}, we use documents from Wikipedia (2018, December) as our retrieval datastore and use BM25 \citep{BM25} as the retriever that pre-selects the 32 most relevant passages for each training query.
\looseness=-1

We compare our method with the following baselines, whose algorithms are detailed in Appendix \ref{sec:app:rag}: (1) Frozen reranker that yields zero-shot performance; (2) The classic ICRALM objective (Eq. (\ref{eq:classic_rag})) \citep{in-context,lewis2020retrieval,realm}, which performs deterministic top-$k$; (3) REPLUG \citep{replug}, which matches the reranker's score distribution with LM’s perplexity distribution conditioned on each selected document (Eq. (\ref{eq:replug})); (4) Policy-only \citep{basu2024statistical}, which samples i.i.d. for $k$ times from the policy defined on `items' and updates with the policy gradient (Eq. (\ref{eq:statistical_framework})); (5) Supervised-only, which is a new method in RAG and a special case of our method with $\lambda=0$.
In Figure \ref{fig:rag}, we report the test performance on Wikitext-103 (in perplexity) and NQ and TriviaQA (in exact match accuracy). We observe more performance gains as $m$ increases, and our method outperforms all the baselines under all $m\in\{8,16,32\}$ for both Qwen-2.5-1.5B and Llama-3.2-3B. Furthermore, we evaluate the document recall on NQ and TriviaQA (Figure \ref{fig:rag_recall}), observing up to a 3.5\%  increase from our method compared to methods in existing literature.


\begin{figure}[h!]
    \centering
    \includegraphics[width=0.95\linewidth]{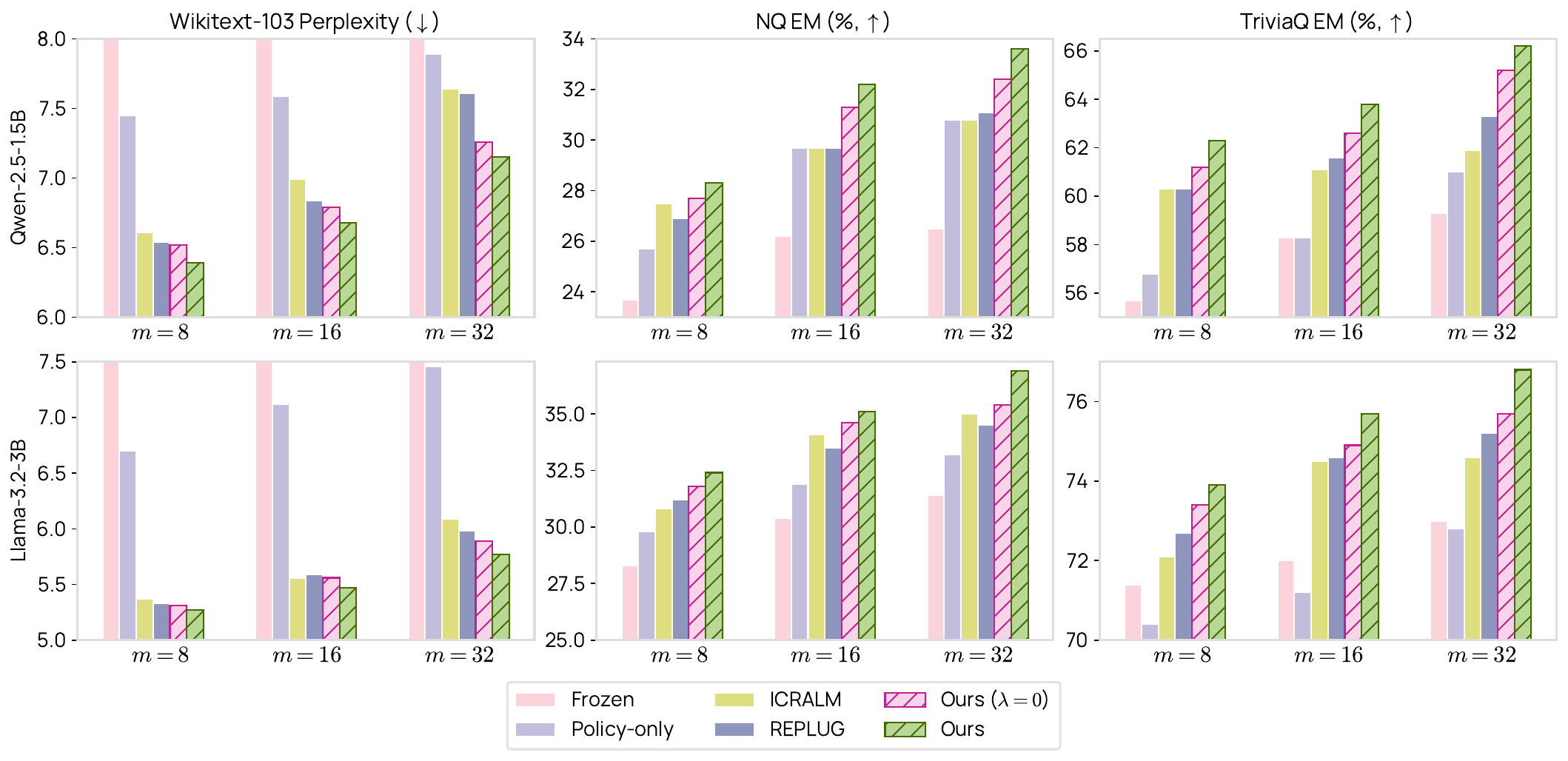}
    \vspace{-1em}
    \caption{Performance of retrieval-augmented Qwen-2.5-1.5B and Llama-3.2-3B with the reranker trained via different methods. We observe improved performance as $m$ increases, and our method outperforms all the baselines. The detailed metric values are reported in Tables \ref{tab:rag_wiki}-\ref{tab:rag_tq_recall}.}
    \label{fig:rag}
    \vspace{-1em}
\end{figure}

\section{Expert Selection for Mixture of Experts} \label{sec:moe}
\vspace{-0.5em}

In this section, we apply our method to MoE models, which contain multiple MoE layers and each layer has one router and $m$ experts. In the $l$-th MoE layer, we have an input $x^{(l)} \in \mathbb{R}^{d^{(l)}}$ where $d^{(l)}$ is the hidden dimension, the experts $h_1^{(l)}, \dots, h_m^{(l)}$ parameterized by weights $\phi^{(l)}$, and a router $\theta^{(l)}$. Given input $x^{(l)}$, the router outputs score $s_{\theta^{(l)}}(i|x^{(l)})$ for all experts $i\in[m]$. Denote $I^{(l)}$ as the sampled set of the indices of the selected unique experts from policy $\pi_{\theta^{(l)}}$. Then this MoE layer outputs
\looseness=-1
$$y^{(l)} = \sum_{i\in I^{(l)}} \frac{\exp({s_{\theta^{(l)}}(i|x^{(l)})})}{\sum_{j\in I^{(l)}} \exp(s_{\theta^{(l)}}(j|x^{(l)}))} \; h_i^{(l)}(x^{(l)}). $$
Therefore, for any sample $(x,y)$ where $x$ is the input sequence and $y$ is the target token, we optimize the router parameters $\theta \coloneqq \{\theta^{(1)}, \ldots, \theta^{(L)} \}$ and expert parameters $\phi \coloneqq \{\phi^{(1)}, \ldots, \phi^{(L)} \}$ for a total of $L$ MoE layers via minimizing the loss Eq. \eqref{eq:our_obj_multilayer}. We have that for an arbitrary layer $l$,
{\setlength{\abovedisplayskip}{4pt}
    \setlength{\belowdisplayskip}{4pt}
    \small
\begin{align*}
    \nabla_{\theta^{(l)}} J_x(\theta) = \mathbb{E}_{I^{(1:L)}} \left[ ( \mathcal{L}(y, x, I^{(1:L)}) -b_x) \cdot \sum_{n=l}^L \nabla_{\theta^{(n)}} \log \pi_{\theta^{(n)}}(I^{(n)}|x^{(n)}) + \nabla_{\theta^{(l)}} \mathcal{L}(y, x, I^{(1:L)}) \right].
\end{align*}}
That is, to compute an unbiased estimate of the gradient, we only need to sample $I^{(1)}\sim\pi_{\theta^{(1)}}$, $\ldots$, $I^{(L)}\sim\pi_{\theta^{(L)}}$ and compute one forward pass. We present the full proof and the multi-layer version of Algorithm \ref{alg:MoT} as Algorithm \ref{alg:multilater} in Appendix \ref{sec:app:method}.

\textbf{Data, models, and baselines.} We train MoE models of nine architectures on Wikitext-103 \citep{wikitext103} from scratch, including the total expert numbers $m\in \{8,16,32,64,128\}$ with one or two MoE layers. To ensure that our observed improvements are systematic, we evaluate our method in a over-training setting (820M tokens, with around 230 tokens-per-activated-parameter) to mitigate spurious gains, providing a rigorous setup for algorithmic comparison.

We compare our method with the same baselines in Section \ref{sec:regression}, whose methods and hyperparameters are tuned based on Appendix \ref{sec:app:exp}: (1) Deterministic top-$k$;  (2) Gaussian MoE \citep{outrageously}; (3) Default MoE \citep{panda2025dense}; (4) SparseMixer \citep{liu2023sparse}; (5) Supervised-only \citep{supervisedonly} (our method with $\lambda=0$). We observe that scaling up $m$ improves model performance, and our method outperforms all the baselines under all $m$. In addition, our method is related to Supervised-only, which is our special case, and SparseMixer, which approximates the policy gradient with high-order derivatives to reduce variance. We show that our $\lambda$ is an effective knob for bias/variance tradeoff: the best $\lambda\in(0,1)$ outperforms Supervised-only ($\lambda=0$) by 2.0 points and outperforms the unbiased update rule ($\lambda=1$) by 7.0 points in perplexity in the case of `128 choose 3', and this observation is consistent in other $m$'s, as presented in Figure \ref{fig:moe_ablation}.

\begin{figure}[h!]
    \centering
    \includegraphics[width=0.52\linewidth, valign=t]{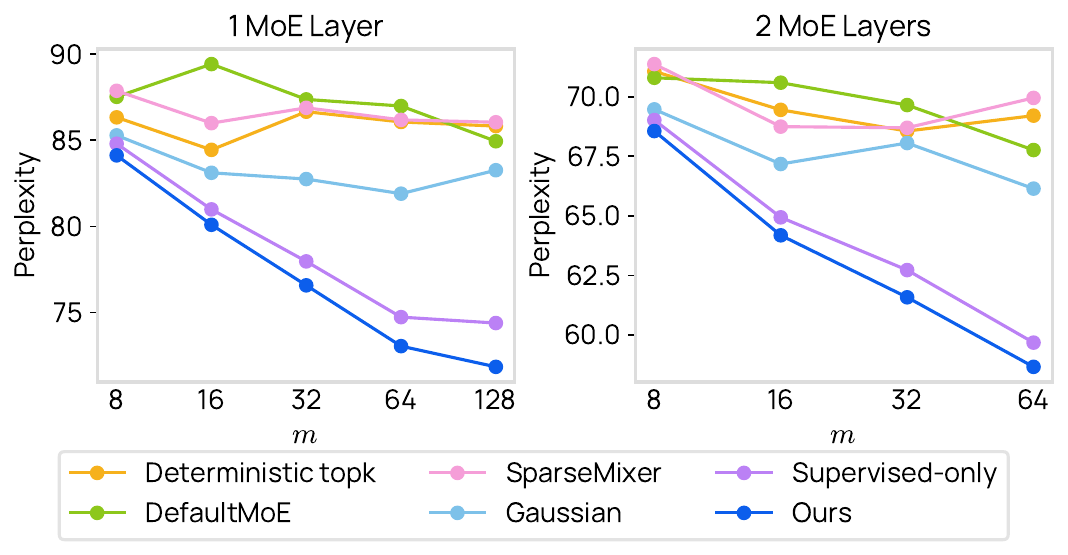}
    \hspace{2.5em}
    \includegraphics[width=0.28\linewidth, valign=t]{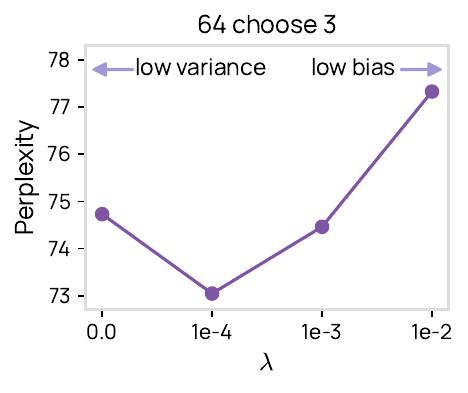}
    \vspace{-0.5em}
    \caption{Test perplexity of different methods after training from scratch on Wikitext-103. Left: Our method outperforms baselines under different numbers of total experts. Right: The best bias/variance tradeoff is achieved by adjusting $\lambda$. The exact numerical values are reported in Table \ref{tab:moe_wiki}.}
    \vspace{-0.5em}
    \label{fig:moe_performance}
\end{figure}


\vspace{-0.5em}
\section{Conclusion and Limitations} \label{sec:conclusion}
\vspace{-0.5em}

We study the problem of top-$k$ selection and propose a principled surrogate objective that models subset selection and upper-bounds the true objective under certain conditions. Our unbiased gradient naturally decomposes into two complementary components: the supervised gradient improves aggregation within the selected subset, and the policy gradient provides direct learning signals on which items to select. 
We provide convergence results, and demonstrate consistent improvements upon strong baselines across regression, RAG, and MoE training in language modeling and QA tasks. 
We identify two limitations of this work. First, while our analysis characterizes the bias/variance tradeoff, selecting $\lambda$ remains problem-dependent in practice. Second, our experiments are conducted on specific model sizes and datasets, and we leave further evaluation at larger scales to future work.

\clearpage
\bibliographystyle{unsrtnat}
\bibliography{cite}

\clearpage
\appendix
\section{Method} \label{sec:app:method}

\subsection{Gradient Derivation}

\paragraph{Case of $L=1$.}
\begin{align*}
    \nabla_{\theta} J_x(\theta,\phi) &= \nabla_{\theta} \sum_{I} \pi_{\theta}(I|x) \mathcal{L}(y, x, I;\theta,\phi) \\
    &= \sum_{I}  \nabla_{\theta} \pi_{\theta}(I|x)  \mathcal{L}(y, x, I;\theta,\phi)  + \pi_{\theta}(I|x) \nabla_{\theta} \mathcal{L}(y, x, I;\theta,\phi) \\
    &\overset{(a)}{=} \sum_{I} \pi_{\theta}(I|x) \nabla_{\theta} \log \pi_{\theta}(I|x) \mathcal{L}(y,x, I;\theta,\phi) + \pi_{\theta}(I|x) \nabla_{\theta} \mathcal{L}(y,x, I;\theta,\phi) \\
    &= \mathbb{E}_{I \sim \pi_{\theta}} \left[ \mathcal{L}(y, x, I;\theta,\phi) \cdot \nabla_{\theta} \log \pi_{\theta}(I|x) + \nabla_{\theta} \mathcal{L}(y, x, I;\theta,\phi) \right] \\
    &\overset{(b)}{=} \mathbb{E}_{I \sim \pi_{\theta}} \left[ \underbrace{(\mathcal{L}(y, x, I;\theta,\phi)-b_x) \cdot \nabla_{\theta} \log \pi_{\theta}(I|x)}_{\text{policy gradient}} + \underbrace{\nabla_{\theta} \mathcal{L}(y, x, I;\theta,\phi)}_{\text{supervised gradient}} \right].
\end{align*}
where step $(a)$ applies the log-derivative trick $\nabla_\theta \pi_\theta = \pi_\theta \nabla_\theta \log \pi_\theta$ and step $(b)$ uses leave-one-out average batch loss as the baseline $b_x$ for any $x$. Abbreviating $\ell_x:=\mathcal{L}(y, x, I;\theta,\phi) $, step $(b)$ holds due to
\begin{align*}
    & \quad \mathbb{E}_{I \sim \pi_{\theta}} \left[ (\ell_{x}-b_x) \cdot \nabla_{\theta} \log \pi_{\theta}(I|x)\right] \\
    & = \mathbb{E}_{I \sim \pi_{\theta}} \left[ \ell_{x} \cdot \nabla_{\theta} \log \pi_{\theta}(I|x)\right] - \mathbb{E}_{I \sim \pi_{\theta}} \left[ \left(\frac{1}{|\mathcal{B}|-1} \sum_{x'\in\mathcal{B},x'\neq x} \ell_{x'}\right) \cdot \nabla_{\theta} \log \pi_{\theta}(I|x)\right] \\
    & = \mathbb{E}_{I \sim \pi_{\theta}} \left[ \ell_{x} \cdot \nabla_{\theta} \log \pi_{\theta}(I|x)\right] - \left(\frac{1}{|\mathcal{B}|-1} \sum_{x'\in\mathcal{B},x'\neq x} \ell_{x'}\right) \underbrace{\mathbb{E}_{I \sim \pi_{\theta}} \left[ \nabla_{\theta} \log \pi_{\theta}(I|x)\right]}_{=0} \\
    & = \mathbb{E}_{I \sim \pi_{\theta}} \left[ \ell_{x} \cdot \nabla_{\theta} \log \pi_{\theta}(I|x)\right].
\end{align*}

\paragraph{Case of $L>1$.}

For any sample $(x,y)$ where $x$ is the input sequence and $y$ is the target token, we optimize the router parameters $\theta \coloneqq \{\theta^{(1)}, \ldots, \theta^{(L)} \}$ and expert parameters $\phi \coloneqq \{\phi^{(1)}, \ldots, \phi^{(L)} \}$ for a total of $L$ MoE layers via minimizing the loss
\begin{align*}
J_x(\theta,\phi) =  \mathbb{E}_{I^{(1)},\ldots,I^{(L)}} \left[ \mathcal{L}(y, x, I^{(1)}, \ldots, I^{(L)}|\theta,\phi) \right].
\end{align*}
We denote $I^{(1)}\ldots I^{(L)}$ as $I^{(1:L)}$ and $\mathcal{L}(\ldots|\theta,\phi)$ as $\mathcal{L}(\ldots)$ for brevity. By the law of iterated expectation,
$\mathbb{E}_{I^{(1:L)}}[\cdot] = \mathbb{E}_{I^{(1)}}\left[\mathbb{E}_{I^{(2)}|I^{(1)}}\left[\cdots \mathbb{E}_{I^{(L)}|I^{(1:L-1)}}[\cdot]\right]\right]$. Thus, for an arbitrary layer $l$,
\begin{align*}
    & \quad \nabla_{\theta^{(l)}} J_x(\theta) \\
    &= \nabla_{\theta^{(l)}} \mathbb{E}_{I^{(1:l-1)}}\left[ \mathbb{E}_{I^{(l)}|I^{(1:l-1)}}\left[ \underbrace{ \mathbb{E}_{I^{(l+1:L)}|I^{(1:l)}}[ \mathcal{L}(y, x, I^{(1:L)}) ]}_{g(I^{(1:l)})}\right] \right] \\
    &= \mathbb{E}_{I^{(1:l-1)}}\left[ \nabla_{\theta^{(l)}} \;\mathbb{E}_{I^{(l)}|I^{(1:l-1)}}\left[ g(I^{(1:l)}) \right] \right] \\
    &\overset{(a)}{=} \mathbb{E}_{I^{(1:l-1)}}\left[ \sum_{I^{(l)}} \nabla_{\theta^{(l)}} \left( \pi_{\theta^{(l)}}(I^{(l)}|x^{(l)}) g(I^{(1:l)}) \right) \right] \\
    &= \mathbb{E}_{I^{(1:l-1)}}\left[ \sum_{I^{(l)}}  \nabla_{\theta^{(l)}} \pi_{\theta^{(l)}}(I^{(l)}|x^{(l)})  g(I^{(1:l)}) + \pi_{\theta^{(l)}}(I^{(l)}|x^{(l)}) \nabla_{\theta^{(l)}} g(I^{(1:l)}) \right] \\
    &= \mathbb{E}_{I^{(1:l)}} \left[ g(I^{(1:l)}) \cdot \nabla_{\theta^{(l)}} \log \pi_{\theta^{(l)}}(I^{(l)}|x^{(l)}) + \nabla_{\theta^{(l)}} g(I^{(1:l)}) \right].
\end{align*}
where $(a)$ is due to $p(I^{(l)} | I^{(1:l-1)}; \theta^{(1:l-1)}) = \pi_{\theta^{(l)}}(I^{(l)} | x^{(l)})$. That is, under the given $\theta^{(1:l-1)}$, the mapping from the past choices $I^{(1:l-1)}$ to the current state $x^{(l)}$ is deterministic.

In the above derivation, note that
\begin{align*}
    g(I^{(1:l)}) & = \mathbb{E}_{I^{(l+1:L)}} \left[ \mathcal{L}(y,x, I^{(1:L)}) | I^{(1:l)} \text{ is sampled} \right] \\
    & = \sum_{I^{(l+1:L)}} p(I^{(l+1:L)}| I^{(1:l)}) \mathcal L(y,x, I^{(1:L)})
\end{align*}
is the expected loss $\mathcal{L}$ from the perspective of layer $l$, after the router has chosen the set $I^{(l)}$, but before the future layers have made their stochastic choices. Denoting $\ell_x:= \mathcal L(y,x, I^{(1:L)})$, we have
$$p(I^{(l+1:L)}| I^{(1:l)}) = \prod_{n=l+1}^L \pi_{\theta^{(n)}} \left(I^{(n)}| x^{(n)}\right);$$
$$\nabla_{\theta^{(l)}}g(I^{(1:l)}) =  \mathbb E_{I^{(l+1:L)}| I^{(1:l)}}\Bigg[\nabla_{\theta^{(l)}}\ell_x+\ell_x \sum_{n=l+1}^{L} \nabla_{\theta^{(l)}} \log\pi_{\theta^{(n)}}(I^{(n)}| x^{(n)}) \Bigg].$$
Therefore, the full derivative is
\begin{align*}
    & \quad \nabla_{\theta^{(l)}} J_x(\theta) \\
    & = \mathbb{E}_{I^{(1:L)}} \left[\ell_x  \cdot \nabla_{\theta^{(l)}} \log \pi_{\theta^{(l)}}(I^{(l)}|x^{(l)})\right] + \mathbb{E}_{I^{(1:L)}} \left[ \nabla_{\theta^{(l)}}\ell_x +\ell_x \sum_{n=l+1}^{L} \nabla_{\theta^{(l)}} \log\pi_{\theta^{(n)}}(I^{(n)}| x^{(n)})\right] \\
    & = \mathbb{E}_{I^{(1:L)}} \left[  \underbrace{\mathcal{L}(y, x, I^{(1:L)})  \cdot \sum_{n=l}^{L}\nabla_{\theta^{(l)}} \log \pi_{\theta^{(n)}}(I^{(n)}|x^{(n)})}_{\text{policy gradient}} + \underbrace{\nabla_{\theta^{(l)}} \mathcal{L}(y, x, I^{(1:L)})}_{\text{supervised gradient}} \right] \numberthis \label{eq:multilater}.
\end{align*}
That is, an unbiased estimate of the gradient is composed of the policy part and the supervised part, and it can be obtained by sampling $I^{(1)}\sim\pi_{\theta^{(1)}}, \ldots, I^{(L)}\sim\pi_{\theta^{(L)}}$ and computing one forward pass. We can again add leave-one-out average batch loss as baseline $b_x$ for any $x$ while keeping the the gradient unbiased: Abbreviating $\ell_{x,I^{(1:L)}} := \mathcal{L}(y, x, I^{(1:L)})$, for $\forall \ n: l\leq n\leq L$, we have
\begin{align*}
    & \quad \mathbb{E}_{I^{(1:L)}} \left[ b_x \cdot \nabla_{\theta^{(l)}} \log \pi_{\theta^{(n)}}(I^{(n)}|x^{(n)}) \right] \\
    & = \mathbb{E}_{I^{(1:L)}} \left[ \left(\frac{1}{|\mathcal{B}|-1} \sum_{x'\in\mathcal{B},x'\neq x} \ell_{x',I^{(1:L)}} \right) \cdot \nabla_{\theta^{(l)}} \log \pi_{\theta^{(n)}}(I^{(n)}|x^{(n)}) \right] \\
    & = \left(\frac{1}{|\mathcal{B}|-1} \sum_{x'\in\mathcal{B},x'\neq x} \ell_{x',I^{(1:L)}} \right) \mathbb{E}_{I^{(1:L)}}\left[ \nabla_{\theta^{(l)}} \log \pi_{\theta^{(n)}}(I^{(n)}|x^{(n)}) \right]
\end{align*}
where
\begin{align*}
    \mathbb{E}_{I^{(1:L)}}\left[ \nabla_{\theta^{(l)}} \log \pi_{\theta^{(n)}}(I^{(n)}|x^{(n)}) \right] &= \mathbb{E}_{I^{(1:n-1)}}\left[ \mathbb{E}_{I^{(n)}|I^{(1:n-1)}}\left[ \nabla_{\theta^{(l)}} \log \pi_{\theta^{(n)}}(I^{(n)}|x^{(n)}) \right] \right] \\
    & = \mathbb{E}_{I^{(1:n-1)}}\left[ \sum_{I^{(n)}}   \pi_{\theta^{(n)}}(I^{(n)}|x^{(n)}) \nabla_{\theta^{(l)}}\log \pi_{\theta^{(n)}}(I^{(n)}|x^{(n)})  \right] \\
    & = \mathbb{E}_{I^{(1:n-1)}}\left[ \sum_{I^{(n)}}  \nabla_{\theta^{(l)}}  \pi_{\theta^{(n)}}(I^{(n)}|x^{(n)}) \right] \\
    & = \mathbb{E}_{I^{(1:n-1)}}\left[\nabla_{\theta^{(l)}} \left( \sum_{I^{(n)}}   \pi_{\theta^{(n)}}(I^{(n)}|x^{(n)}) \right)\right] \\
    &= 0.
\end{align*}
We thus have Algorithm \ref{alg:multilater} for the case of $L>1$. We note that by summing the policy gradient from all layers $l\in[L]$, we only need to run backpropagation to the loss in Line 10 once.  

\begin{algorithm}
\caption{Training the Selection and Prediction Module with Hybrid Gradients ($L>1$)}
\label{alg:multilater}
\begin{algorithmic}[1]
\State \textbf{Input:} Selection parameters $\theta$, prediction parameters $\phi$, $\eta$, $\lambda$, $\tau$
\For{each batch $\mathcal{B}$}
    \For{$(x,y)\in \mathcal{B}$}
        \State $p_{\theta^{(l)}}(\cdot|x) \gets \text{softmax}(s_{\theta^{(l)}}(\cdot|x)/ \tau)$ for all $l\in[L]$
        \State Sample $I^{(l)} \sim \text{Multinomial}(p_{\theta^{(l)}}(\cdot|x))$ without replacement for all $l\in[L]$
        \State Compute $\log \pi_{\theta^{(l)}}(I^{(l)}|x^{(l)})$ by Eq. (\ref{eq:unordered}) for all $l\in[L]$
        \State $\ell_{x} \gets \mathcal{L}(y,x,I^{(1:L)}|\theta^{(1:L)},\phi^{(1:L)})$
    \EndFor
    \State Compute $b_x \gets$ leave-one-out loss average for each $x\in\mathcal{B}$ by Eq. \eqref{eq:baseline}
    \State $\ell_\mathcal{B}
    \gets \frac{\lambda}{|\mathcal{B}|} \sum_{x\in \mathcal{B}} \sum_{l\in [L]} \text{SG}\left(\ell_x - b_x\right) \cdot \log \pi_{\theta^{(l)}}(I^{(l)}|x^{(l)}) + \frac{1}{|\mathcal{B}|} \sum_{x\in \mathcal{B}} \ell_{x}$ 
    \State $[\theta,\phi] \gets [\theta,\phi] - \eta\nabla_{[\theta,\phi]} \ell_\mathcal{B}$
\EndFor
\end{algorithmic}
\end{algorithm}

\subsection{Selection Module Policy}

\begin{lemma}
    Sampling $I$ from $\pi_{\theta}$ in Eq. (\ref{eq:unordered}) can be implemented by \texttt{L=multinomial($p_{\theta}(\cdot|x)$, $k$, replacement=False)} and then using $L$ as a set. In other words, if $\pi(I)$ is the probability of $I$ where $I$ is an unordered version of an ordered sequence, then sampling an ordered sequence and then forgetting the order follows exactly $\pi(I)$.
\end{lemma}
\begin{proof}
    Let the sampled ordered sequence be $l = \{l_1, \ldots, l_k\}$ and denote its unordered set version as $i=\text{set}(l)$. Denote $\text{perm}(i)$ as a set containing all permutations (ordered sequences) of the elements in $i$. Then by the definition of marginalization,
    \begin{align*}
        P(I=i|x) = \sum_{l'\in \text{perm}(i)} p(l'_1, \ldots, l'_k|x) \overset{(\ref{eq:unordered})}{=} \pi_{\theta} (I=i|x).
    \end{align*}
\end{proof}

\subsection{Recipe for Selecting $\lambda$}

Theoretically, the optimal $\lambda^{\star}$ value would be small if the policy gradient variance bound $\sigma_P^2$ is large (Theorem \ref{thm:convergence}). Empirically, however, selecting $\lambda$ remains problem-dependent. To reduce the tuning workload, we study how to quickly determine whether a $\lambda$ value is too big or too small. 

We compare the gradient norm and the loss standard deviation of the supervised component and the policy component, respectively. We observe that both metrics are informative indicators of appropriate $\lambda$ values: As shown in Figure \ref{fig:grad_norm}, large $\lambda$ leads to an overwhelming policy gradient norm and policy loss standard deviation compared to the supervised gradient. In contrast, small $\lambda$ yields a small contribution from the policy component. Under appropriate $\lambda$, the gradient norm and loss standard deviation from the two components are comparable.

\begin{figure}[h!]
    \centering
    \includegraphics[width=0.85\linewidth, valign=t]{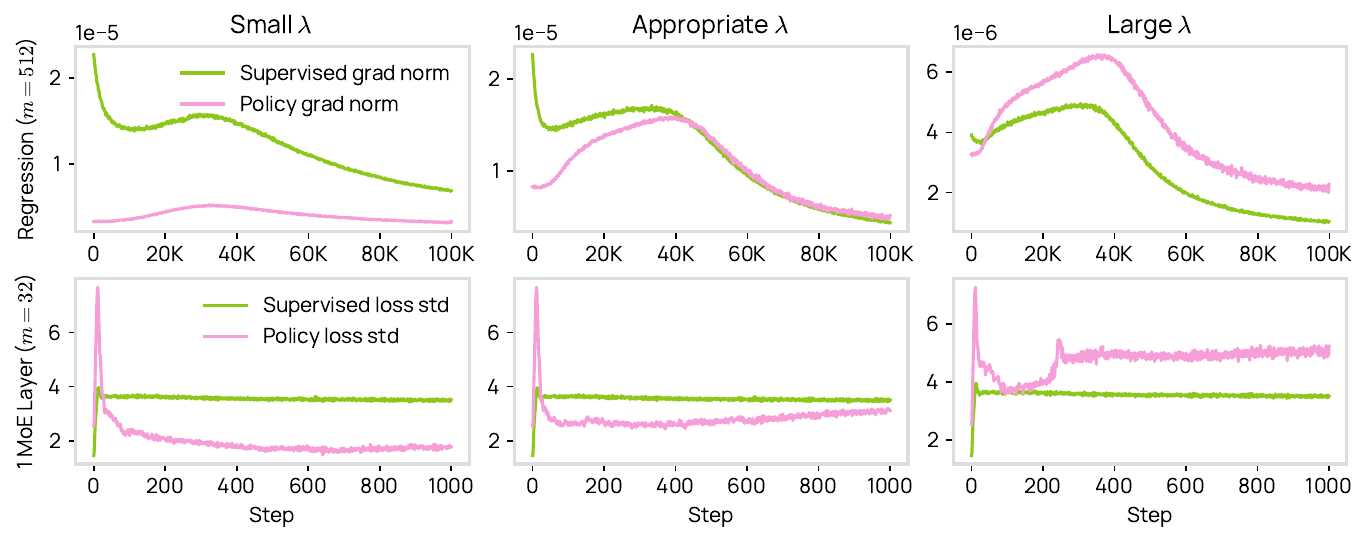}
    \caption{Gradient norm and loss standard deviation of the supervised component and the policy component. Under appropriate $\lambda$, the two metrics for the two components are comparable.}
    \label{fig:grad_norm}
\end{figure}

Based on this observation, we provide a recipe for quickly searching for the largest appropriate $\lambda$: First, pick a $\lambda$ and examine the first few steps of the optimization; then reduce $\lambda$ by 10$\times$ each time if the policy gradient norm or loss standard deviation overwhelms that of the supervised component.
\section{Convergence Analysis} \label{sec:app:theory}

This section proves the convergence of our method to our objective and the true objective. In the following, we first present the problem formulation and prove the useful lemmas.

\subsection{Problem Setup}

Let the training set be $\mathcal D=\{(x_n,y_n)\}_{n=1}^N$. We train prediction module parameters $\phi$ jointly with the selection module parameters $\theta$ and let $w:=(\theta,\phi)$. The selection module $s_\theta:\mathbb{R}^{h}\to\mathbb{R}^m$ produces scores $z_i := s_{\theta}(x)_i, i\in[m]$.  The empirical true objective is
$$ F(\theta,\phi) := \frac1N\sum_{n=1}^N \mathcal L\left( y_n, \sum_{i=1}^m \frac{e^{s_{\theta}(x_n)_i}}{\sum_{j=1}^m e^{s_{\theta}(x_n)_j}} h_{i,\phi}(x_n) \right). $$

\paragraph{Our surrogate.}
For each training sample $x_n$, let $I\sim \pi_\theta(\cdot| x_n) $ be the sampled subset. Define
$$ \ell_n(\theta,\phi,I) := \mathcal L\left( y_n, \sum_{i\in I} \frac{e^{s_{\theta}(x_n)_i}}{\sum_{j\in I}e^{s_{\theta}(x_n)_j}} h_{i,\phi}(x_n) \right) $$
and our surrogate objective is
$$ J(\theta,\phi) := \frac1N\sum_{n=1}^N \E_{I\sim\pi_\theta(\cdot| x_n)} [ \ell_n(\theta,\phi,I) ]. $$

Recall that the gradient of $J$ with respect to $\theta$ decomposes into the policy gradient $P_\theta(\theta,\phi)$ and the supervised gradient $S_\theta(\theta,\phi)$, i.e., $\nabla_\theta J(\theta,\phi)=P_\theta(\theta,\phi)+S_\theta(\theta,\phi)$ where
$$P_\theta(\theta,\phi):=\frac1N\sum_{n=1}^N \E_{I\sim\pi_\theta(\cdot| x_n)} [\ell_n(\theta,\phi,I)\nabla_\theta\log\pi_\theta(I| x_n)],$$
$$S_\theta(\theta,\phi):=\frac1N\sum_{n=1}^N \E_{I\sim\pi_\theta(\cdot| x_n)} [\nabla_\theta\ell_n(\theta,\phi,I)].$$

Since $\pi_\theta$ does not depend on $\phi$,
$$ \nabla_\phi J(\theta,\phi)= S_\phi(\theta,\phi) =\frac1N\sum_{n=1}^N \E_{I\sim\pi_\theta(\cdot| x_n)}[ \nabla_\phi\ell_n(\theta,\phi,I) ]. $$

\paragraph{Minibatch update.}

At iteration $t$, we sample a minibatch $\mathcal B_t$ i.i.d. uniformly from $\mathcal{D}$ with replacement. For each $b\in\mathcal B_t$, we sample $I_{b}\sim \pi_{\theta_t}(\cdot| x_{b})$. Using a leave-one-out batch baseline,
$$ b_{-b} := \frac{1}{B-1}\sum_{a\neq b}\ell_{a}(\theta_t,\phi_t,I_{a}),$$
we use the selection-model update
$$ g_{\theta,t}^{\lambda} = \frac{1}{B}\sum_{b=1}^B
\left[ \lambda \bigl( \ell_{b}(\theta_t,\phi_t,I_{b})-b_{-b} \bigr) \nabla_\theta\log\pi_{\theta_t}(I_{b}| x_{b})+\nabla_\theta \ell_{b}(\theta_t,\phi_t,I_{b})\right].$$
and the prediction-model update
$$ g_{\phi,t} = \frac1B\sum_{b=1}^B \nabla_\phi \ell_b.$$
The full-model gradient is denoted as $g_t^{\lambda}:= (g_{\theta,t}^{\lambda},g_{\phi,t})$, with which we conduct the joint update $w_{t+1} = w_t-\eta g_t^{\lambda}$. Note that the joint gradient estimator is related to the gradient of our objective via
$$ \E[g_t^{\lambda}| w_t]= \nabla_w J(w_t)+ \tilde b_t^{\lambda}, \quad \tilde b_t^{\lambda}=\bigl(-(1-\lambda)P_\theta(w_t),0\bigr).$$

Relative to the true objective $F$, our update satisfies $ g_t^{\lambda}=\nabla_w F(w_t)+ b_t^{\lambda}+\xi_t,$ where $\xi_t$ is the gradient variance with $\E[\xi_t| w_t]=0$ and
$$b_t^{\lambda}=\underbrace{\nabla_w J(w_t)-\nabla_w F(w_t)}_{\text{truncation bias}}+\underbrace{\bigl(-(1-\lambda)P_\theta(w_t),0\bigr)}_{\text{policy gradient suppression bias}}.$$

In the following, we first derive the lemmas showing that the constants $\sigma_S^2$ (supervised gradient variance bound), $\sigma_P^2$ (policy gradient variance bound), $\epsilon$ (upper-bound of the gradient bias $\|\nabla_w J - \nabla_w F\|$ from objective truncation), and $R_{\max}$ (the maximal full-batch policy gradient norm on the optimization trajectory) are finite under Assumptions \ref{ass:compact-domain}$-$\ref{ass:model-smoothness}.

\subsection{Lemmas}


\begin{lemma}[Derived constants]
\label{lem:derived-constants}
Under Assumptions~\ref{ass:compact-domain}--\ref{ass:model-smoothness}, the following quantities are finite:
\begin{enumerate}[label=(\alph*)]
    \item \textbf{Prediction module output.} $C_h := \max_{i\in[m]} \sup_{w\in\Omega, \|x\|\le C_x}\|h_{i,\phi}(x)\| < \infty.$
    \item \textbf{Prediction module Jacobian.} $C_{h'} := \max_{i\in[m]} \sup_{w\in\Omega, \|x\|\le C_x}\|\nabla_\phi h_{i,\phi}(x)\| < \infty.$
    \item \textbf{Selection module Jacobian.} $C_r := \max_{i\in[m]} \sup_{w\in\Omega, \|x\|\le C_x}\left\|\frac{\partial z_i}{\partial\theta}(x;\theta)\right\| < \infty.$
    \item \textbf{Score function.} $C_\pi := \sup_{w\in\Omega, \|x\|\le C_x, I: |I|=k}\|\nabla_\theta\log\pi_\theta(I| x)\| < \infty.$
    \item \textbf{Loss value.} $\ell_{\max} := \sup_{w\in\Omega, \|x\|\le C_x, I: |I|=k} \ell_n(\theta,\phi,I) <\infty$.
    \item \textbf{Loss variance.} $\sigma_l^2 := \sup_{w\in\Omega} \operatorname{Var}_{n,I}(\ell_n(\theta,\phi,I))\le\sigma_l^2 <\infty$.
    \item \textbf{Policy gradient.} $R_{\max} := \sup_{w\in \Omega}\|P_\theta(w)\| <\infty$.
\end{enumerate}
\end{lemma}

\begin{proof}
Parts $(a)$--$(c)$ follow from the same argument: each quantity is the supremum of a continuous function over the compact set $\Omega\times\{x:\|x\|\le C_x\}$. By Assumption~\ref{ass:model-smoothness}, $h_{i,\phi}(x)$, $\nabla_\phi h_{i,\phi}(x)$, and $\partial z_i/\partial\theta$ are all continuous in $(w,x)$. A continuous function on a compact set attains its supremum, so each quantity is finite.

To prove Part $(d)$, recall that 
$$\pi_\theta(I|x)= \sum_{L\in\mathrm{perm}(I)} \prod_{t=1}^k \frac{p_{\theta}(L_t|x)}{1-\sum_{s<t}p_{\theta}(L_s|x)},\qquad p_{\theta}(L_t|x)=\mathrm{softmax}(z/\tau) .$$
By the chain rule,
$$\nabla_\theta\log\pi_\theta(I|x)= \sum_{i=1}^m \frac{\partial\log\pi_\theta(I|x)}{\partial z_i} \frac{\partial z_i}{\partial\theta}.$$
Hence,
$$\|\nabla_\theta\log\pi_\theta(I|x)\|\le \sum_{i=1}^m\left| \frac{\partial\log\pi_\theta(I|x)}{\partial z_i} \right| \left\| \frac{\partial z_i}{\partial\theta} \right\| \le C_r\sum_{i=1}^m \left| \frac{\partial\log\pi_\theta(I|x)}{\partial z_i} \right|$$
by Part $(c)$. To bound $\sum_i |\partial\log\pi_\theta/\partial z_i|$, we note that
$$\frac{\partial p_{\theta}(i|x)}{\partial z_j} = \frac{1}{\tau}p_{\theta}(i|x)(\delta_{ij}-p_{\theta}(j|x)), \qquad \left|\frac{\partial p_{\theta}(i|x)}{\partial z_j}\right| \le \frac{p_{\theta}(i|x)}{\tau}.$$

Write
$$\pi_\theta(I|x)=\sum_{L\in\mathrm{perm}(I)}\prod_{t=1}^k q_t(L), \qquad q_t(L)=\frac{p_{\theta}(L_t|x)}{1-\sum_{s<t}p_{\theta}(L_s|x)}.$$
Then
$$\frac{\partial\log\pi_\theta}{\partial z_j} = \frac{1}{\pi_\theta} \sum_{L\in\mathrm{perm}(I)} \prod_{t=1}^k q_t(L)  \sum_{t=1}^k \frac{1}{q_t(L)} \frac{\partial q_t(L)}{\partial z_j}.
$$
For each factor $q_t$, by the quotient rule,
$$\frac{\partial q_t}{\partial z_j} = \frac{1}{1-\sum_{s<t}p_{\theta}(L_s|x)} \frac{\partial p_{\theta}(L_t|x)}{\partial z_j} + \frac{p_{\theta}(L_t|x)}{\left(1 -\sum_{s<t}p_{\theta}(L_s|x)\right)^2} \sum_{s<t}\frac{\partial p_{\theta}(L_s|x)}{\partial z_j}.$$
Since $L_t$ is still available at stage $t$, $1-\sum_{s<t}p_{\theta}(L_s|x)\ge p_{\theta}(L_t|x)$. Using $\left|\partial p_{\theta}(i|x)/\partial z_j\right|\le p_{\theta}(i|x)/\tau$, we obtain
$$\left| \frac{1}{q_t}\frac{\partial q_t}{\partial z_j} \right| \le \frac{1}{\tau}+\frac{t-1}{\tau} \le \frac{k}{\tau}.$$
Therefore,
$$\left|\frac{\partial\log\pi_\theta}{\partial z_j}\right| \le \sum_{t=1}^k \frac{k}{\tau} = \frac{k^2}{\tau}.$$
Summing over $j=1,\dots,m$ gives
$$\sum_{j=1}^m \left| \frac{\partial\log\pi_\theta}{\partial z_j} \right| \le \frac{mk^2}{\tau}.$$
Hence,
$$\|\nabla_\theta\log\pi_\theta(I|x)\| \le C_\pi := \frac{mk^2 C_r}{\tau}.$$

For Part $(e)$--$(f)$, we leverage the bounded domain, prediction modules, and contexts assumptions, and the Lipschitz loss $\mathcal{L}$ assumption, which gives that the per-sample loss is uniformly upper-bounded: $\ell_n(\theta,\phi,I)\leq \ell_{\max}, \forall w \in \Omega, n\in[|\mathcal{D}|], I $, where $\ell_{\max}$ depends on $\bar L$, $C_h$, and $\Omega$. Therefore, there exists a finite constant $\sigma_l^2$ such that $\operatorname{Var}_{n,I}(\ell_n(\theta,\phi,I)) \leq \sigma_l^2$.

For Part $(g)$, note that for any $w\in\Omega$,
\begin{align*}
\|P_\theta(w)\|
&= \left\|\frac{1}{N}\sum_{n=1}^{N} \E_{I\sim\pi_\theta(\cdot|x_n)} \bigl[\ell_n(\theta,\phi,I) \nabla_\theta\log\pi_\theta(I|x_n)\bigr]\right\| \\
&\le \frac{1}{N}\sum_{n=1}^{N} \E_{I}\bigl[|\ell_n(\theta,\phi,I)| \|\nabla_\theta\log\pi_\theta(I|x_n)\|\bigr] \\
&\le \frac{1}{N}\sum_{n=1}^{N} \ell_{\max} C_\pi\\
& = \ell_{\max} C_\pi.
\end{align*}
Therefore, we take $R_{\max} := \ell_{\max} C_\pi < \infty$.
\end{proof}

\begin{lemma}[Smoothness and boundedness]
Both $J$ and $F$ are bounded below and $\beta$-smooth for some finite $\beta$.
\end{lemma}

\begin{proof}
$J(w) = \frac{1}{|\mathcal{D}|} \sum_{(x,y)\in\mathcal{D}}\sum_{I}\pi_\theta(I|x)\mathcal{L}(y,x,I;w)$, which contains a finite sum over $\binom{m}{k}$ subsets. Each term is smooth in $w$: (i) $\pi_\theta(I|x)$ is smooth in $\theta$ since $\mathrm{softmax}(z/\tau)$ is composed with the smooth selection module $z = s_\theta(x)$; (ii) $ \mathcal{L}(y,x,I;w)$ is smooth in $w$ since $\mathcal{L}$ is smooth, before which the weighted aggregations (both $h_{i,\phi},i\in[m]$ and $s_{\theta}$) are smooth. Therefore, $\nabla^2 J$ exists and is continuous on $\Omega$, hence bounded: 
$\|\nabla^2 J\| \leq \beta$ for some $\beta < \infty$. The same argument applies to 
$F(w) = \frac{1}{|\mathcal{D}|} \sum_{(x,y)\in\mathcal{D}}\mathcal{L}(y, \sum_{i\in[m]} p_{\theta}(i|x)h_{i,\phi}(x))$. 
\end{proof}

\begin{lemma}[Bounded supervised gradient variance]
\label{lem:bounded_supervised_var}
There exists $\sigma_S^2<\infty$ such that the minibatch variance of the joint supervised gradient estimator satisfies
$$\E\left[\left\|\frac1B\sum_{b=1}^B
\begin{pmatrix}
\nabla_\theta \ell_b\\
\nabla_\phi \ell_b
\end{pmatrix}
- \begin{pmatrix} S_\theta(w)\\ 
S_\phi(w) \end{pmatrix}
\right\|^2\middle| w\right]\le\sigma_S^2.$$
\end{lemma}

\begin{proof}
The supervised gradient for a single sample is $(\nabla_\theta\ell, \nabla_\phi\ell)$. The minibatch variance is $\frac{1}{B}\mathrm{Var}_{n,I}((\nabla_\theta\ell, \nabla_\phi\ell))$, which is bounded if each per-sample gradient is bounded. 

\paragraph{Term $\nabla_\theta\ell$.} By the chain rule, $\nabla_\theta\ell = \nabla_{\hat{y}}\mathcal{L}\cdot\nabla_\theta\hat{y}$. Since $\mathcal{L}(y,\hat y)$ is $\bar L$-Lipschitz, $\|\nabla_{\hat{y}}\mathcal{L}\| \leq \bar L$. For $\nabla_\theta\hat{y}$, the output $\hat{y} = \sum_{i\in I}p_ih_i$ where $p_i = e^{z_i}/\sum_{j\in I}e^{z_j}$. For $l \in I$, the derivative with respect to $z_l$ is the standard softmax Jacobian that
$$\frac{\partial p_i}{\partial z_l} = p_i(\delta_{il} - p_l).$$
By chain rule,
$$\nabla_\theta p_i = \sum_{l \in I} \frac{\partial p_i}{\partial z_l} \cdot \frac{\partial z_l}{\partial\theta} = \sum_{l \in I} p_i(\delta_{il} - p_l) \cdot \frac{\partial z_l}{\partial\theta},$$
we take norms and obtain
$$\|\nabla_\theta p_i\| \leq \sum_{l \in I} p_i|\delta_{il} - p_l| \cdot C_r = p_i C_r \sum_{l \in I}|\delta_{il} - p_l|.$$
To compute $\sum_{l \in I}|\delta_{il} - p_l|$, we split into $l = i$ and $ l \neq i$ cases, so
$$\sum_{l \in I}|\delta_{il} - p_l| = |1 - p_i| + \sum_{l \in I, l \neq i} p_l = (1 - p_i) + (1 - p_i) = 2(1 - p_i).$$
Therefore, $\|\nabla_\theta p_i\| \leq 2C_r p_i(1 - p_i) \leq 2C_r p_i$ and
$$\|\nabla_\theta\hat{y}\| = \| \sum_{i\in I} \nabla_{\theta} p_i h_i\| \leq \sum_{i\in I} \|\nabla_{\theta} p_i\| \|h_i\|\leq 2C_r C_h.$$

This yields $\|\nabla_\theta\ell\| \leq 2\bar LC_rC_h$.

\paragraph{Term $\nabla_\phi\ell$.} By the chain rule,
$\nabla_\phi\ell = \nabla_{\hat{y}}\mathcal{L}\cdot\nabla_\phi\hat{y} = \nabla_{\hat{y}}\mathcal{L}\cdot\sum_{i\in I}p_i \nabla_\phi h_i(x)$. By Lemma \ref{lem:derived-constants} Part $(b)$, $\|\nabla_\phi h_i(x)\|$ is upper-bounded by finite $C_{h'}$, so $\|\nabla_\phi\ell\|\leq \bar LC_{h'}.$
Combining terms yields 
$$ \|\nabla_{w}\ell\| \leq \|\nabla_{\theta}\ell\|+\|\nabla_{\phi}\ell\| \leq \bar LC_{h'} + 2\bar L C_r C_h < \infty,$$
so the minibatch variance is bounded:
$$ \frac{1}{B}\text{Var}_{n,I}(\nabla_w \ell) \leq \frac{1}{B}\mathbb{E}[\|\nabla_w \ell\|^2] \leq \sigma_S^2 := \frac{(\bar L C_{h'} + 2\bar L C_hC_r)^2}{B}.$$

\end{proof}

\begin{lemma}[Bounded policy gradient variance]
\label{lem:bounded_policy_var}
Denote $\widehat P_B$ as the batch policy gradient estimator. Then there exists $\sigma_P^2<\infty$ such that the minibatch variance of the policy gradient estimator satisfies
$$ \E\|\widehat P_B-P_\theta\|^2\le \sigma_P^2.$$
\end{lemma}

\begin{proof}
Denote
$$ Y_b:=\ell_{n_b}(\theta,\phi,I_b), \qquad s_b:=\nabla_\theta\log\pi_\theta(I_b|x_{n_b}), \qquad b_{-b}:=\frac{1}{B-1}\sum_{a\ne b}Y_a$$
and then our policy gradient estimator is $\widehat P_B = \frac1B\sum_{b=1}^B (Y_b-b_{-b})s_b.$
Let $\mu:=\E_{n,I}[Y]$ and $X_b:=Y_b-\mu$. Since $\E[s_b]=0$, subtracting the constant $\mu$ does not change the estimator, i.e.,
$$ (Y_b-b_{-b})s_b = \left( X_b-\frac1{B-1}\sum_{a\ne b}X_a \right)s_b.$$
Therefore,
$$ \widehat P_B = U-V, \quad \text{where} \quad U:=\frac1B\sum_{b=1}^B X_b s_b, \quad V:= \frac{1}{B(B-1)} \sum_{b=1}^B \sum_{a\ne b}X_a s_b. $$
In addition, $ P_\theta(w)=\E[Xs]$ since $\E[\mu s]=0$. Hence, $ \widehat P_B-P_\theta = (U-P_\theta)-V$ and 
$$ \E\|\widehat P_B-P_\theta\|^2 \le 2\E\|U-P_\theta\|^2 + 2\E\|V\|^2.$$

For the first component,
$$ \E\|U-P_\theta\|^2 = \frac1B\operatorname{Var}(Xs) \le \frac1B\E[X^2\|s\|^2] \le \frac{C_\pi^2}{B}\E[X^2].$$
By Lemma \ref{lem:derived-constants} Part $(f)$, $ \E[X^2]=\operatorname{Var}_{n,I}(\ell_n(\theta,\phi,I)) \le \sigma_l^2$, so we have $ \E\|U-P_\theta\|^2 \le \frac{C_\pi^2}{B}\sigma_l^2.$

For the second component, we expand $\E\|V\|^2$ and the only nonzero cross terms are identical-pair and swapped-pair terms. Using $\E[X]=0$, $\E[s]=0$, independence across minibatch samples, and $\|s\|\le C_\pi$, we obtain
$$ \E\|V\|^2 \le \frac{2C_\pi^2}{B(B-1)} \operatorname{Var}_{n,I}(\ell_n(\theta,\phi,I))\le\frac{2C_\pi^2}{B(B-1)}\sigma_l^2.$$
Combining the two components, we have
$$\E\|\widehat P_B-P_\theta\|^2\le\frac{2C_\pi^2}{B}\sigma_l^2+\frac{4C_\pi^2}{B(B-1)}\sigma_l^2 \leq \sigma_P^2:=\frac{4C_\pi^2}{B}\left(1+\frac1{B-1}\right)\sigma_l^2.$$
\end{proof}

\begin{lemma}[Master descent lemma]
Let $\phi$ be a $\beta$-smooth function and bounded below by $\phi^{\star}$. Consider $w_{t+1} = w_t-\eta g_t$ with $ g_t = \nabla \phi(w_t)+b_t+\xi_t$ where $ \E[\xi_t|w_t]=0$, $\E[\|\xi_t\|^2|w_t]\le \sigma^2$, $\eta\le \frac1{2\beta}.$ Then
$$ \frac1T\sum_{t=0}^{T-1} \E\|\nabla \phi(w_t)\|^2 \le \frac{2(\phi(w_0)-\phi^{\star})}{\eta T} + \beta\eta\sigma^2 + \frac3T\sum_{t=0}^{T-1}\|b_t\|^2.$$
\end{lemma}

\begin{proof}
By $\beta$-smoothness,
$$ \phi(w_{t+1}) \le \phi(w_t) + \langle\nabla \phi(w_t),w_{t+1}-w_t\rangle + \frac{\beta}{2}\|w_{t+1}-w_t\|^2.$$
Substituting $ w_{t+1}-w_t = -\eta(\nabla \phi(w_t)+b_t+\xi_t)$, taking conditional expectation, and using $\E[\xi_t|w_t]=0$ gives
\begin{align*}
    \E[\phi(w_{t+1})|w_t] & \le \phi(w_t) -\eta\|\nabla \phi\|^2 -\eta\langle\nabla \phi,b_t\rangle+\frac{\beta\eta^2}{2}\left(\|\nabla \phi+b_t\|^2+\sigma^2\right) \\
    & \le \phi(w_t)-\eta\left(1-\frac{\beta\eta}{2}\right)\|\nabla \phi\|^2-\eta(1-\beta\eta)\langle\nabla \phi,b_t\rangle+\frac{\beta\eta^2}{2}(\|b_t\|^2+\sigma^2)
\end{align*}
by expanding the square. Since $\eta\le 1/(2\beta)$, we have $1-\frac{\beta\eta}{2}\ge \frac34$ and $|1-\beta\eta|\le 1.$ Then using Young's inequality $|\langle\nabla \phi,b_t\rangle|\le\frac14\|\nabla \phi\|^2+\|b_t\|^2$ yields
$$\E[\phi(w_{t+1})|w_t]\le \phi(w_t) -\frac{\eta}{2}\|\nabla \phi(w_t)\|^2 + \frac{3\eta}{2}\|b_t\|^2 + \frac{\beta\eta^2}{2}\sigma^2.$$
Rearranging, summing over $t=0,\ldots,T-1$, telescoping, and dividing by $\eta T/2$ gives the claim.
\end{proof}

\subsection{Convergence to Our Objective}
 Let Assumptions \ref{ass:compact-domain}-\ref{ass:model-smoothness} hold. Then $J$ is $\beta$-smooth for some finite $\beta$ and bounded below by $J^{\star}$. For any fixed $\lambda$, running Algorithm \ref{alg:MoT} with $\eta=\frac{1}{\beta\sqrt T}$ and $T\ge 4$ yields
$$ \frac1T\sum_{t=0}^{T-1} \E\|\nabla J(w_t)\|^2 \le \frac{2\beta\Delta_0+2\sigma_S^2 + 2\lambda^2\sigma_P^2}{\sqrt T} + 3(1-\lambda)^2R_{\max}^2,$$
where $\Delta_0:=J(w_0)-J^{\star}$. For any finite $T$, the upper bound is minimized at
$$ \lambda^\star(T) = \frac{3R_{\max}^2} {2\sigma_P^2/\sqrt T+3R_{\max}^2},$$
which gives
$$ \frac1T\sum_{t=0}^{T-1} \E\|\nabla J(w_t)\|^2 \le \frac{2\beta\Delta_0+2\sigma_S^2}{\sqrt T} + \frac{6\sigma_P^2R_{\max}^2} {2\sigma_P^2+3R_{\max}^2\sqrt T}.$$

\begin{proof}

Recall that the joint estimator satisfies $\E[g_t^{\lambda}|w_t] =\nabla_w J(w_t)+b_t^{\lambda}$ where $b_t^{\lambda}=\bigl(-(1-\lambda)P_\theta(w_t),\,0\bigr).$ With $ R_{\max}:=\sup_{w\in w}\|P_\theta(w)\|$, we have $\|b_t^{\lambda}\| \le (1-\lambda)R_{\max}.$

Let $ \xi_t := g_t^{\lambda} - \E[g_t^{\lambda}|w_t]$. Since the estimator decomposes into the joint supervised gradient minibatch noise and the scaled policy gradient noise, by $\|a+b\|^2\le 2\|a\|^2+2\|b\|^2$, we have $ \E[\|\xi_t\|^2|w_t] \leq 2\sigma_S^2 + 2\lambda^2\sigma_P^2$.

Applying the master descent lemma with $\phi=J$, 
$\sigma^2=2\sigma_S^2+2\lambda^2\sigma_P^2$ and $\|b_t^{\lambda}\|^2 \le (1-\lambda)^2R_{\max}^2$ yields
$$ \frac1T\sum_{t=0}^{T-1} \E\|\nabla J(w_t)\|^2
\le \frac{2\beta\Delta_0+2\sigma_S^2 + 2\lambda^2\sigma_P^2}{\sqrt T}
+ 3(1-\lambda)^2R_{\max}^2.$$
The $\lambda$-dependent part is
$$ g(\lambda) = \frac{2\lambda^2\sigma_P^2}{\sqrt T} + 3(1-\lambda)^2R_{\max}^2.$$
We set $g'(\lambda)=0$ and obtain
$$ \lambda^\star(T) = \frac{3R_{\max}^2}{2\sigma_P^2/\sqrt T+3R_{\max}^2}.$$
Since $ g''(\lambda) =\frac{4\sigma_P^2}{\sqrt T}+ 6R_{\max}^2>0$, $\lambda^{\star}$ is the unique minimizer. Applying $\lambda^{\star}$ to $g(\lambda)$ gives
$$ g(\lambda^\star) = \frac{6\sigma_P^2R_{\max}^2}{2\sigma_P^2+3R_{\max}^2\sqrt T}.$$
Adding the $\lambda$-independent term $ \frac{2\beta\Delta_0+2\sigma_S^2}{\sqrt T}$ gives the claim.
\end{proof}

\subsection{Convergence to the True Objective}

\begin{lemma}[Jacobian mismatch]
\label{lem:jacobian-mismatch}
Denote $p=\text{softmax}(z/\tau)$. For any subset $I$ with $|I|=k$ and $\rho_I:=\sum_{i\notin I}p_i\leq \frac12$, we have
$$\|\nabla_\theta\hat f (I,x;\theta)-\nabla_\theta f(x;\theta)\| \leq C_B \rho_I, \quad \text{where}$$
$$C_B:=12C_rC_h, \qquad f=\sum_i p_i h_i,
\qquad \hat f =\sum_{i\in I}\frac{p_i}{1-\rho_I}h_i.$$
\end{lemma}

\begin{proof}
Note that $\nabla_\theta p_i=\sum_{l=1}^m p_i(\delta_{il}-p_l)\nabla_\theta z_l$ and thus $\|\nabla_\theta p_i\| \leq p_i\sum_l|\delta_{il}-p_l| C_r.$ Since
$$\sum_l|\delta_{il}-p_l| = (1-p_i)+\sum_{l\neq i}p_l = (1-p_i)+(1-p_i) = 2(1-p_i) \leq 2,$$
we obtain 
\begin{align}
\label{eq:softmax-jacobian-bound}
\|\nabla_\theta p_i\|\leq 2C_r p_i \quad\text{for all }i.
\end{align}
Note that
$$\nabla_\theta\hat f  = \frac{1}{1-\rho_I}\sum_{i\in I}\nabla_\theta p_i\cdot h_i + \frac{\nabla_\theta\rho_I}{(1-\rho_I)^2}\sum_{i\in I}p_i h_i, \quad \text{and}$$
$$\nabla_\theta f = \sum_{i\in I}\nabla_\theta p_i\cdot h_i + \sum_{i\notin I}\nabla_\theta p_i\cdot h_i.$$
Subtracting them yields 
$$\nabla_\theta\hat f -\nabla_\theta f = \underbrace{\frac{\rho_I}{1-\rho_I}\sum_{i\in I}\nabla_\theta p_i\cdot h_i}_{(\mathrm{i})} + \underbrace{\frac{\nabla_\theta\rho_I}{(1-\rho_I)^2}\sum_{i\in I}p_i h_i}_{(\mathrm{ii})} - \underbrace{\sum_{i\notin I}\nabla_\theta p_i\cdot h_i}_{(\mathrm{iii})}.$$

For term $(\mathrm{i})$, using Eq. \eqref{eq:softmax-jacobian-bound} gives
$$\sum_{i\in I}\|\nabla_\theta p_i\|C_h \leq 2C_r C_h\sum_{i\in I}p_i = 2C_rC_h(1-\rho_I),$$
so
$$\|(\mathrm{i})\| \leq \frac{\rho_I}{1-\rho_I}\cdot 2C_rC_h(1-\rho_I) = 2C_rC_h\rho_I.$$

For term $(\mathrm{ii})$, by Eq. \eqref{eq:softmax-jacobian-bound},
$$\|\nabla_\theta\rho_I\| = \left\|\sum_{i\notin I}\nabla_\theta p_i\right\| \leq 2C_r\sum_{i\notin I}p_i = 2C_r\rho_I.$$
Also,
$$\left\|\sum_{i\in I}p_ih_i\right\| \leq (1-\rho_I)C_h \leq C_h.$$
Therefore,
$$\|(\mathrm{ii})\| \leq \frac{2C_r\rho_I}{(1-\rho_I)^2}\cdot C_h = \frac{2C_rC_h\rho_I}{(1-\rho_I)^2}.$$
For $\rho_I\leq 1/2$, we have $(1-\rho_I)^2\geq 1/4$, hence $\|(\mathrm{ii})\|\leq 8C_rC_h\rho_I.$

For term $(\mathrm{iii})$,
$$\|(\mathrm{iii})\| \leq \sum_{i\notin I}2C_rp_i\cdot C_h = 2C_rC_h\rho_I.$$

Combining the three bounds,
$$\|\nabla_\theta\hat f -\nabla_\theta f\| \leq (2+8+2)C_rC_h\rho_I = 12C_rC_h\rho_I = C_B\rho_I.$$
\end{proof}

\begin{lemma}[$\theta$-gradient truncation bias]
$$\|b_t^{\mathrm{trunc}}\| = \|\nabla_{\theta} J(\theta)-\nabla_{\theta} F(\theta)\| \leq c_{\theta}(m-k)\left(e^{-\gamma_{\min}}+e^{-\gamma_{\min}/\tau}\right) =:\epsilon_{\theta},$$
where $c_{\theta}$ depends on
$\bar L,\bar\beta,C_h,C_r,m,k,\tau$.
\end{lemma}

\begin{proof}
Denote $\Delta L_I:=\mathcal{L}(y,\hat f )-\mathcal{L}(y,f)$. By the score-function identity,
$$b_t^{\mathrm{trunc}} = \frac{1}{|\mathcal{D}|} \sum_{(x,y)\in\mathcal{D}} \mathbb{E}_{I\sim\pi_\theta} \left[ \Delta L_I\cdot\nabla_\theta\log\pi_\theta(I|x) + \nabla_\theta\Delta L_I \right].$$

\paragraph{Term of $\mathbb{E}_I[\Delta L_I\cdot\nabla_\theta\log\pi_\theta]$.}
By Lipschitzness of $\mathcal{L}$ and Lemma \ref{lem:jacobian-mismatch}, $|\Delta L_I| \leq \bar L\|\hat f -f\| \leq 2\bar L C_h\rho_I$. Since $\|\nabla_\theta\log\pi_\theta\|\leq C_\pi$, we have
\begin{equation}
\label{eq:T1}
\|\mathbb{E}_I[\Delta L_I\cdot\nabla_\theta\log\pi_\theta]\|
\leq 2\bar L C_hC_\pi\cdot\mathbb{E}_I[\rho_I].
\end{equation}

\paragraph{Term of $\mathbb{E}_I[\nabla_\theta\Delta L_I]$.}
By adding and subtracting,
$$\nabla_\theta\Delta L_I = \nabla_\theta \hat f \cdot \nabla_{\hat y} \mathcal{L}(\hat f) - \nabla_{\theta} f \cdot \nabla_{\hat y} \mathcal{L}(f) = \underbrace{[\nabla_{\hat y}\mathcal{L}(\hat f )-\nabla_{\hat y}\mathcal{L}(f)]}_{A} \cdot\nabla_\theta\hat f  + \nabla_{\hat y}\mathcal{L}(f)\cdot \underbrace{[\nabla_\theta\hat f -\nabla_\theta f]}_{B}.$$

For term $A$, by the smoothness of $\mathcal{L}$,
$\|A\| \leq \bar\beta\|\hat f -f\| \leq 2\bar\beta C_h\rho_I$. Also, by the same bounds as in Lemma~\ref{lem:jacobian-mismatch}, $ \|\nabla_\theta\hat f \| \leq 12C_rC_h =:C_B'$. Therefore, $\|A\cdot\nabla_\theta\hat f \| \leq 2\bar\beta C_hC_B'\rho_I$ and hence
\begin{align}
\label{eq:A-term}
\|\mathbb{E}_I[A\cdot \nabla_\theta \hat f ]\| \le 2\bar\beta  C_h C_B' \mathbb{E}_I[\rho_I].
\end{align}

For term $B$, since $\|\nabla_{\hat y}\mathcal{L}(f)\|\le \bar L$, we split the expectation
$$\mathbb{E}_I[\|B\|] = \mathbb{E}_I[\|B\| \mathbf 1\{I=S\}] + \mathbb{E}_I[\|B\| \mathbf 1\{I\neq S\}].$$

By Assumption \ref{ass:compact-domain}, we denote $\gamma_{\min}$ as the minimal logit difference between the selected and unselected item along the optimization trajectory that
$$\mathrm{ess inf}_x\bigl(s_{(k)}(x;\theta)-s_{(k+1)}(x;\theta)\bigr)\geq \gamma_{\min}>0.$$
Therefore, when $I=S$, for sufficiently large $\gamma_{\min}$, we have
$$\rho_S\le (m-k)e^{-\gamma_{\min}}<\frac12,$$
so Lemma~\ref{lem:jacobian-mismatch} applies and yields $\|B\| = \|\nabla_\theta \hat f -\nabla_\theta f\| \le C_B\rho_S$. Hence,
$$\mathbb{E}_I[\|B\| \mathbf 1\{I=S\}] \le C_B\rho_S \mathbb P(I=S).$$

When $I\neq S$, we use $\|B\|\le C_{\bar B}$ where $C_{\bar{B}} := 14\bar L C_r C_h$ due to $\|\nabla_\theta f\| \leq 2C_r C_h$ by Lemma \ref{lem:bounded_supervised_var} and $\|\nabla_\theta \hat{f}\| \leq C'_B = 12 C_r C_h$ by Lemma~\ref{lem:jacobian-mismatch}, which yields
$\|\nabla_{\theta} \hat f - \nabla_{\theta} f \|\leq \|\nabla_{\theta} \hat f \| + \|\nabla_{\theta} f \|\leq 14C_r C_h$. Therefore,
$$\mathbb{E}_I[\|B\| \mathbf 1\{I\neq S\}] \le C_{\bar B} \mathbb P(I\neq S).$$

Combining the two cases, $\mathbb{E}_I[\|B\|] \le C_B\rho_S \mathbb P(I=S) + C_{\bar B} \mathbb P(I\neq S)$. Thus
$$\mathbb E_I \left[\|\nabla_{\hat y}\mathcal L(f)\cdot B\|\right] \le  \bar L  C_B\rho_S \mathbb P(I=S) + L_\mathcal L C_{\bar B} \mathbb P(I\neq S).$$

Combining this with Eq. \eqref{eq:A-term} yields
$$\|\mathbb{E}_I[\nabla_\theta \Delta L_I]\| \le 2\beta_\mathcal L C_h C_B' \mathbb{E}_I[\rho_I] + L_\mathcal L C_B\rho_S \mathbb P(I=S) + L_\mathcal L C_{\bar B} \mathbb P(I\neq S).$$
Using $\rho_S\le (m-k)e^{-\gamma_{\min}}$ and $\mathbb P(I\neq S)\le (m-k)e^{-\gamma_{\min}/\tau}$, we obtain
\begin{equation}
\label{eq:pathwise-bound}
\|\mathbb{E}_I[\nabla_\theta \Delta L_I]\| \le c(m-k)\left(e^{-\gamma_{\min}}+e^{-\gamma_{\min}/\tau}\right)
\end{equation}
for a constant $c>0$ depending on
$L_\mathcal L,\beta_\mathcal L,C_h,C_r,m,k,\tau$.

\paragraph{Combining Term 1 and Term 2.}
From \eqref{eq:T1} and \eqref{eq:pathwise-bound}, we obtain
$$
\|b_t^{\mathrm{trunc}}\| \le 2\bar L C_h C_\pi \mathbb E_x\mathbb{E}_I[\rho_I] + c(m-k)\left(e^{-\gamma_{\min}}+e^{-\gamma_{\min}/\tau}\right).
$$
Using
$$\mathbb{E}_I[\rho_I] \le (m-k)e^{-\gamma_{\min}} + (m-k)e^{-\gamma_{\min}/\tau},$$
we conclude
$$\|b_t^{\mathrm{trunc}}\| \le c_{\theta}(m-k)\left(e^{-\gamma_{\min}}+e^{-\gamma_{\min}/\tau}\right), \quad c_{\theta}:=2\bar L C_h C_\pi + c.$$
\end{proof}

\begin{lemma}[$\phi$-gradient truncation bias] 
$$
\|\nabla_\phi J(\theta,\phi)-\nabla_\phi F(\theta,\phi)\| \le c_{\phi}(m-k)\left( e^{-\gamma_{\min}} + e^{-\gamma_{\min}/\tau} \right),$$
where $c_{\phi}$ depends on
$ \bar L ,\bar\beta ,C_h,C_{h'},m,k,\tau$.
\end{lemma}

\begin{proof}

Since $\pi_\theta(I| x)$ does not depend on $\phi$,
$$
\nabla_\phi J(\theta,\phi)-\nabla_\phi F(\theta,\phi)=\E_x\E_{I\sim\pi_\theta}[\nabla_\phi \mathcal L(y,\hat f )-\nabla_\phi \mathcal L(y,f)].
$$
By the chain rule,
$ \nabla_\phi \mathcal L(y,\hat f ) - \nabla_\phi \mathcal L(y,f) = A_I+B_I$, where
$$ A_I := \bigl[ \nabla_{\hat y}\mathcal L(y,\hat f ) - \nabla_{\hat y}\mathcal L(y,f) \bigr] \nabla_\phi \hat f, \quad B_I := \nabla_{\hat y}\mathcal L(y,f)
\bigl[\nabla_\phi \hat f - \nabla_\phi f\bigr].$$

By the smoothness of $\mathcal L$, $ \|\nabla_{\hat y}\mathcal L(y,\hat f ) - \nabla_{\hat y}\mathcal L(y,f)\| \le \bar\beta \|\hat f -f\|.$ Also, $ \|\hat f -f\| \le 2C_h\rho_I.$ and 
$ \nabla_\phi \hat f  = \sum_{i\in I} \frac{p_i}{1- \rho_I} \nabla_\phi h_{i,\phi}(x),$
so $ \|\nabla_\phi \hat f \| \le C_{h'}$. Therefore,
$\|A_I\| \le 2\bar\beta C_hC_{h'} \rho_I$. Since
$$ \nabla_\phi f = \sum_{i=1}^m p_i\nabla_\phi h_{i,\phi}(x), \quad \nabla_\phi \hat f  = \sum_{i\in I} \frac{p_i}{1-\rho_I} \nabla_\phi h_{i,\phi}(x),$$
we have
$$
\nabla_\phi \hat f -\nabla_\phi f
=\frac{\rho_I}{1-\rho_I}\sum_{i\in I}p_i\nabla_\phi h_{i,\phi}(x)-\sum_{i\notin I}p_i\nabla_\phi h_{i,\phi}(x).$$
Using $\|\nabla_\phi h_{i,\phi}(x)\|\le C_{h'}$, we get
$$\left\|\frac{\rho_I}{1-\rho_I}\sum_{i\in I}p_i\nabla_\phi h_{i,\phi}(x)\right\|\le\rho_I C_{h'}, \quad \left\|\sum_{i\notin I}p_i\nabla_\phi h_{i,\phi}(x)\right\|\le\rho_I C_{h'}.$$
Therefore, $\|\nabla_\phi \hat f -\nabla_\phi f\|\le2C_{h'}\rho_I.$ Since $\|\nabla_{\hat y}\mathcal L(y,f)\|\le  \bar L ,$ we have
$\|B_I\|\le2 \bar L C_{h'}\rho_I.$

Combining the two bounds,
$$\|\nabla_\phi \mathcal L(y,\hat f )-\nabla_\phi \mathcal L(y,f)\|\le 2C_{h'}(\bar\beta C_h+ \bar L )\rho_I.$$
Taking expectation over $I$ and $x$,
$$\|\nabla_\phi J(\theta,\phi)-\nabla_\phi F(\theta,\phi)\| \le 2C_{h'}(\bar\beta C_h+ \bar L ) \E_x\E_I[\rho_I].$$
Using $\E_I[\rho_I]\le(m-k)e^{-\gamma_{\min}}+(m-k)e^{-\gamma_{\min}/\tau},$ we obtain
$$\|\nabla_\phi J(\theta,\phi)-\nabla_\phi F(\theta,\phi)\|\le c_{\phi}(m-k)\left(e^{-\gamma_{\min}}+e^{-\gamma_{\min}/\tau}\right), \quad c_{\phi}:=2C_{h'}(\bar\beta C_h+ \bar L ).$$

\end{proof}

\begin{lemma}[Truncation bias]
\label{lem:truncation-bias}
We combine the $\theta$-gradient and $\phi$-gradient truncation bias and have
$$ \|\nabla_w J - \nabla_w F\|^2 = \| \nabla_{\theta}J - \nabla_{\theta}F \|^2 + \| \nabla_{\phi}J - \nabla_{\phi}F \|^2$$
and thus
$$ \epsilon := \sqrt{\epsilon_{\theta}^2 + \epsilon_{\phi}^2} \leq \epsilon_{\theta} + \epsilon_{\phi} $$
\end{lemma}

With the above lemmas derived, we prove Theorem \ref{thm:convergence} as follows.

\begin{proof}

Recall that relative to the true objective $F$, $g_t^{\lambda}=\nabla_w F(w_t)+ b_t^{\lambda}+\xi_t$, where
$$b_t^{\lambda}=\underbrace{\nabla_w J(w_t)-\nabla_w F(w_t)}_{\text{truncation bias}}+\underbrace{\bigl(-(1-\lambda)P_\theta(w_t),0\bigr)}_{\text{policy gradient suppression bias}}.$$

Define $\epsilon := \sup_t
\|\nabla_w J(w_t)-\nabla_w F(w_t)\| $ and $R_{\max} := \sup_t\|P_\theta(w_t)\|$,
then $ \|b_t^{\lambda}\| \le \epsilon+(1-\lambda)R_{\max}.$ With $\epsilon$ derived in Lemma~\ref{lem:truncation-bias} and $R_{\max}$ derived in Lemma \ref{lem:derived-constants}, we apply the master descent lemma with with $\phi=F$, $\sigma^2=2\sigma_S^2+2\lambda^2\sigma_P^2$, and $ \|b_t^{\lambda}\|^2
\leq \left(\epsilon+(1-\lambda)R_{\max}\right)^2 $. With step size $\eta=\frac{1}{\beta\sqrt{T}}$ and $ \eta\leq \frac{1}{2\beta}$ (i.e. $T\geq 4$),
we have
$$\frac{1}{T}\sum_{t=1}^T\mathbb{E}\|\nabla F(w_t)\|^2 \leq \frac{2\beta(F(w_0)-F^\star)+2\sigma_S^2+2\lambda^2\sigma_P^2}{\sqrt{T}} + 3\left(\epsilon+(1-\lambda)R_{\max}\right)^2 .$$
For any fixed $T$, the bound is minimized at
\begin{align*}
    \lambda^\star(T) &= \arg\min_{\lambda\in[0,1]} \left[ \frac{2\lambda^2\sigma_P^2}{\sqrt{T}} + 3(\epsilon+(1-\lambda)R_{\max})^2 \right] \\
                    & = \min\left\{1,~ \frac{3R_{\max}(\epsilon+R_{\max})} {2\sigma_P^2/\sqrt{T}+3R_{\max}^2}\right\}.
\end{align*}
\end{proof}

The gradient of our objective provides a smooth bias-variance tradeoff controlled by $\lambda$: Increasing $\lambda$ reduces the bias $ b_t^{\lambda}=\nabla_w J(w_t)-\nabla_w F(w_t)-\bigl((1-\lambda)P_\theta(w_t),0\bigr)$ and therefore lowers the asymptotic bias floor in the stationarity bound; At the same time, increasing $\lambda$ also increases the variance of the estimator through the policy gradient term, which enters the bound as $ 2\lambda^2\sigma_P^2/\sqrt{T}$. 
\section{Experiments} \label{sec:app:exp}

In the following, we detail the data, baselines, hyperparameters, and ablation results for each application. In Figure \ref{fig:convergence}, we visualize the convergence of all methods in different applications. Our experiments are conducted on a cluster containing A40, L40S, and A100 GPUs. 

\subsection{Regression Task} \label{sec:app:toy}

\paragraph{Data generation.} Recall that $\{v_1,\ldots,v_m\}$ are $m$ candidate vectors of dimension $d$, where we take $d=128,k=3$, and $m\in\{64,128,256,512\}$. We sample each coordinate of $x\in\mathbb{R}^d$ following $\mathcal{N}(0,1)$, each coordinate of $v$ following $\mathcal{N}(0,1/d)$, and each element of $\theta^{\star}$ following $\mathcal{N}(0,1/d)$. Denote $S$ as the set that contains the top-3 elements. Then we have
$$y = \sum_{i\in S} p_i v_i, \quad v_i\sim \mathcal{N}(0,I_h/d).$$
Conditioned on the top-3 weights $p$, we use $\mathbb{E}[\|v_i\|^2]=1$ and obtain
$$\mathbb{E}[\|y\|^2|p] = \sum_{i\in S} p_i^2 \mathbb{E}[\|v_i\|^2] = \sum_{i\in S} p_i^2.$$
Therefore,
$$\mathbb{E}[\|y\|^2]  = \mathbb{E}\left[\sum_{i\in S} p_i^2\right] =  \frac{1}{3} + \frac{1}{9 \log m} + o \left(\frac{1}{\log m}\right)$$
where the second equality is due to Lemma \ref{lem:top3}. That is, we show that as $m$ increases, the expected squared norm of $y$ only slightly decreases. Therefore, our observation that MSE loss increases as $m$ increases indicates the increasing optimization difficulty.

\begin{lemma}[Top-3 softmax weight concentration] \label{lem:top3}
Let $z_1,\dots,z_ m$ be i.i.d. sampled from $\mathcal{N}(0,1)$ and denote the top-3 as $z_{(1)} \ge z_{(2)} \ge z_{(3)}.$ Define the top-3 softmax weights
$$
p_i = \frac{e^{z_{(i)}}}{\sum_{j=1}^3 e^{z_{(j)}}}, \quad i=1,2,3.$$
Then as $ m \to \infty$,
$$ \mathbb{E} \left[\sum_{i=1}^3 p_i^2\right] = \frac{1}{3} + \frac{1}{9 \log  m} + o \left(\frac{1}{\log  m}\right).$$
\end{lemma}

\begin{proof}

Define gaps
$\Delta_2 = z_{(1)} - z_{(2)}$ and $\Delta_3 = z_{(1)} - z_{(3)}$. By shift invariance of softmax,
$$p_1 = \frac{1}{1 + e^{-\Delta_2} + e^{-\Delta_3}}, \quad
p_2 = \frac{e^{-\Delta_2}}{Z}, \quad
p_3 = \frac{e^{-\Delta_3}}{Z},$$
where $Z = 1 + e^{-\Delta_2} + e^{-\Delta_3}$. Denote $\Phi$ as the standard normal CDF, and let
$$b_m = \Phi^{-1} \left(1 - \frac{1}{m}\right) \sim \sqrt{2\log m},$$
then the classical extreme value theory \citep{evt1,evt2} gives $z_{(1)} = b_m + o_p(1)$, and the spacings satisfy
$$\Delta_2 = \frac{H_2}{b_m} + o_p \left(\frac{1}{b_m}\right), \quad \Delta_3 = \frac{H_3}{b_m} + o_p \left(\frac{1}{b_m}\right),$$
where $(H_2,H_3)=O_p(1)$. Therefore, $\Delta_i = O_p(1/b_m)$. Let $\delta_1=0$, $\delta_2=\Delta_2$, $\delta_3=\Delta_3$. We have
$$ p_i = \frac{e^{-\delta_i}}{\sum_{j=1}^3 e^{-\delta_j}}. $$

For small $\delta_i$, expand $ e^{-\delta_i} = 1 - \delta_i + \frac{\delta_i^2}{2} + O(\delta_i^3)$. Let $ \bar{\delta} = \frac{1}{3} \sum_{i=1}^3 \delta_i$, then a standard expansion gives
$$ p_i = \frac{1}{3} - \frac{1}{3}(\delta_i - \bar{\delta}) + O(\|\delta\|^2). $$
Using the expansion,
$$ p_i^2 = \frac{1}{9} - \frac{2}{9}(\delta_i - \bar{\delta}) + \frac{1}{9}(\delta_i - \bar{\delta})^2 + O(\|\delta\|^3). $$

Summing over $i$, we have
$$ \sum_{i=1}^3 p_i^2 = \frac{1}{3} + \frac{1}{9}\sum_{i=1}^3 (\delta_i - \bar{\delta})^2 + O(\|\delta\|^3).$$

By $\delta_i = O_p(1/b_m)$, we have
$$ \sum_{i=1}^3 (\delta_i - \bar{\delta})^2 = \frac{1}{b_m^2} \sum_{i=1}^3 (H_i - \bar{H})^2 + o_p \left(\frac{1}{b_m^2}\right).$$

Therefore, taking expectation yields
$$\mathbb{E} \left[\sum_{i=1}^3 p_i^2\right] = \frac{1}{3}+ \frac{1}{9 b_m^2} \mathbb{E} \left[\sum_{i=1}^3 (H_i - \bar{H})^2\right] + o \left(\frac{1}{b_m^2}\right).$$

For Gaussian extremes, we have
$$\mathbb{E} \left[\sum_{i=1}^3 (H_i - \bar{H})^2\right] =2.$$
Therefore,
$$ \mathbb{E} \left[\sum_{i=1}^3 p_i^2\right] = \frac{1}{3} + \frac{2}{9 b_m^2} + o \left(\frac{1}{b_m^2}\right) = \frac{1}{3} + \frac{1}{9 \log m} + o \left(\frac{1}{\log m}\right)$$
due to $b_m^2 \sim 2\log m$.
\end{proof}


\begin{table}[!h]
\setlength{\tabcolsep}{2.5pt}
\small\centering
\caption{Performance on the regression task.} 
\label{tab:regression_result}
\begin{tabular}{lcccccccc}
\toprule
\multirow{2}{*}{Method} & \multicolumn{4}{c}{MSE Loss ($\times$1e-3, $\downarrow$)} &  \multicolumn{4}{c}{Top-1 Selection Accuracy ($\%,\uparrow$)} \\
 & $m=64$ & $m=128$ & $m=256$ & $m=512$ & $m=64$ & $m=128$ & $m=256$ & $m=512$ \\
\midrule
Deterministic & 1.01 & 2.02 & 5.73 & 5.99 & 89.0 & 66.7 & 4.0 & 0.6 \\
Gaussian & 1.39 & 2.15 & 5.14 & 6.22 & 80.8 & 64.5 & 14.3 & 1.0 \\
DefaultMoE & 1.20 & 1.44 & 1.87 & 3.29 & 80.1 & 68.6 & 53.5 & 41.4 \\
SparseMixer & 1.39 & 1.92 & 3.11 & 4.30 & 86.6 & 73.9 & 46.8 & 27.0 \\
Supervised-only & 0.99 & 1.43 & 1.91 & 4.67 & 89.9 & 80.0 & 61.2 & 11.9 \\
\midrule
Ours & \textbf{0.96} & \textbf{1.17} & \textbf{1.51} & \textbf{2.86} & \textbf{91.0} & \textbf{84.5} & \textbf{66.8} & \textbf{47.6} \\
\bottomrule
\end{tabular}
\end{table}

\begin{figure}[h!]
    \centering
    \includegraphics[width=\linewidth, valign=t]{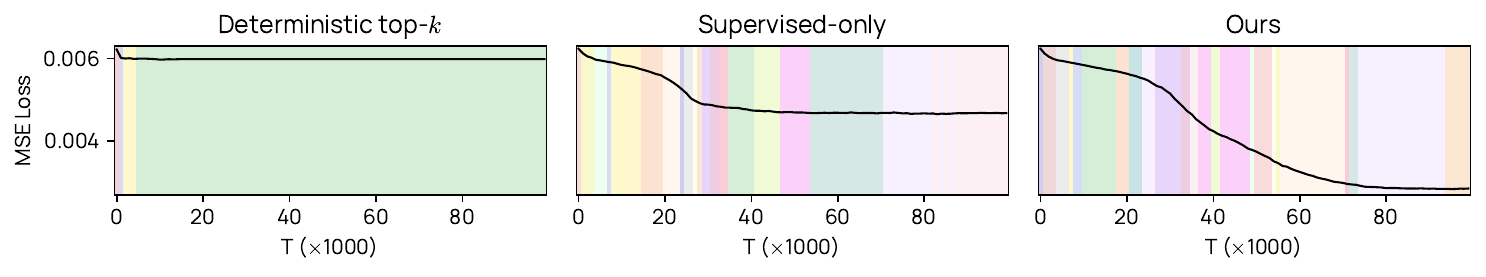}
    \caption{Decrease of test loss and the evolution of selection decisions (on two test samples) during optimization from $\theta_0$ to $\theta_{T}$ under different methods. The line chart represents loss decrease, and each background color represents one selection decision on two test samples. We observe that the selection module explores more selection decisions with our method.}
    \label{fig:1D_landscape}
\end{figure}

\begin{figure}[h!]
    \centering
    \includegraphics[width=\linewidth, valign=t]{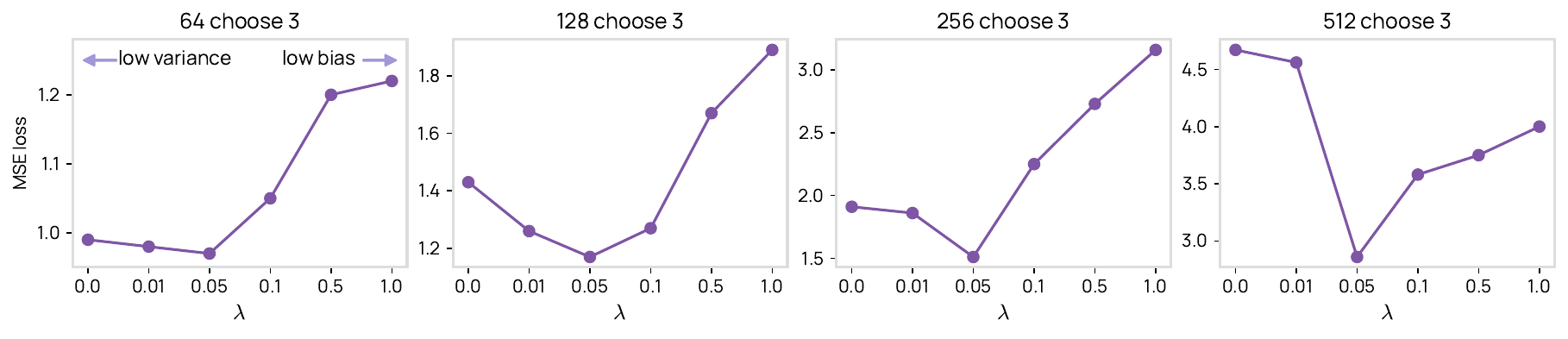}
    \caption{Bias/variance tradeoff adjusted by $\lambda$, showing that the optimal tradeoff is obtained at $\lambda\in(0,1)$ under different $m$'s.}
    \label{fig:regression_ablation}
\end{figure}


\paragraph{Hyperparameters.} For all the methods, we tune the learning rate in the grid \{1e-5, 5e-5, 1e-4, 5e-4, 1e-3\} and use full batch training. For all the methods that involve sampling from the multinomial distribution defined by the softmax of logits, we tune the temperature $\tau\in\{0.02, 0.05, 0.1, 0.5, 1.0, 2.0\}$. For our method, we tune $\lambda\in$\{0.01, 0.05, 0.1, 0.5\}. For the DefaultMoE baseline, we tune $\beta\in$\{0.6, 0.7, 0.8, 0.9, 0.99\}.

\subsection{Retrieval-Augmented Generative Models} \label{sec:app:rag}


\paragraph{Prediction module.} Note that  $p(y|x)$ can be decomposed in terms of latent variables, such as an item $i$ or a subset $I$: $ p(y|x) = \sum_{i} p(y|x,i) \cdot p(i|x) = \sum_{I} p( y|x,I) \cdot p(I|x)$. We denote $ w_i(x)=e^{s(i|x)}$ and $p(i|x)=\frac{w_i(x)}{\sum_{j=1}^m w_j(x)}$. Since we generate an ordered sequence $L=(L_1, \ldots, L_k)$ without replacement and then discard the order and obtain $I=\{L_1,\ldots,L_k\}$, with
$$p(L|x)=\prod_{t=1}^k \frac{w_{L_t}(x)}{\sum_{l\notin \{L_1,\ldots,L_{t-1}\}}w_l(x)}, \quad p(I|x)=\sum_{L:\{L_1,\ldots,L_k\}=I}p(L|x).$$
In the following, we show that taking
$$p(y|x,I) = \sum_{i\in I} a_i(x,I) \cdot p(y|x,i), \quad a_i(x,I) = \frac{\sum_{\substack{L:\{L_1,\ldots,L_k\}=I \\ L_1=i}} p(L|x)}{p(I|x)} $$
makes our objective an upper bound of the true objective. First, the weights $a_i(x,I)$ form a probability distribution within every set since $\sum_{i\in I} a_i(x,I) = 1$. Additionally,
\begin{align*}
    \sum_{I: i\in I} p(I|x) a_i(x,I) & = \sum_{I: i\in I} \sum_{\substack{L:\{L_1,\ldots,L_k\}=I \\ L_1=i}} p(L|x) \\
    & = \sum_{I: i\in I} \mathrm{Pr}(I, L_1=i|x) \\
    & = \mathrm{Pr}(L_1=i|x) \\
    & = p(i|x),
\end{align*}
because the first draw follows our original softmax distribution. Therefore,
\begin{align*}
    \sum_{I} p(I|x) p( y|x,I) &= \sum_I p(I|x) \sum_{i\in I} a_i(x,I) p(y|x,i) \\
    &= \sum_I \sum_{i\in I} p(I|x) a_i(x,I) p(y|x,i) \\
    & = \sum_i p(i|x) p(y|x,i) \\
    & = p(y|x).
\end{align*}
That is, $a_i(x,I), i\in I$ are the correct aggregation weights for $p(y|x,i)$ so that $ p(y|x) = \sum_{i} p(y|x,i) \cdot p(i|x) = \sum_{I} p( y|x,I) \cdot p(I|x)$ is valid.

Next, with the model parameters $\theta$ and $\phi$, denote $\pi_{\theta}(I|x)$ as the probability of sampling the subset $I$ given $x$, and denote $f_{\theta,\phi}(I,x)$ as the probability of outputting $y$ given $x$ and selecting subset $I$. By Jensen's inequality, since $\mathcal{L} = -\log(\cdot)$ which is convex, then for each sample $x$, we have 
\begin{align*}
\mathcal{L} \left(y, \sum_{i=1}^m p_{\theta}(i|x) \cdot  p(y|i,x) \right) & = \mathcal{L} \left(y, \sum_{I} f_{\theta,\phi}(I,x) \cdot \pi_{\theta}(I|x) \right) \\ & \overset{\text{Jensen}}{\leq} \sum_{I} \mathcal{L} (y, f_{\theta,\phi}(I,x)) \cdot \pi_{\theta}(I|x) \\
& = \mathbb{E}_{I\sim\pi_{\theta}} [\mathcal{L} (y, x,I;\theta,\phi) ].
\end{align*}

The left-most expression is the true objective, and the right-most is our proposed objective. Therefore, our objective upper-bounds the true objective when (1) $\mathcal{L}$ is convex and (2) the model output is probability $p(y|I,x)$.

\paragraph{Baselines.} In our RAG experiments, we compare with the following classic baselines. One classic objective to train a reranker \citep{in-context,lewis2020retrieval,realm} uses
\begin{align}
     -\log p(y|x) \approx -\log \left(\sum_{d\in\text{top-}k} p_{\theta}(d|x) \cdot p_{\text{LM}}(y|d,x) \right) \label{eq:classic_rag}\\ 
     \text{where} \quad p_{\theta}(d|x) = \frac{\exp(s_{\theta}(d|x)/\tau)}{\sum_{d'\in\text{top-}k}\exp(s_{\theta}(d'|x)/\tau)}.
\end{align}
Furthermore, REPLUG \citep{replug} proposes minimizing
\begin{align}
    \text{KL} (p_{\theta}(d|x) \Vert q(d|x,y)) \label{eq:replug}\quad  \text{where}
\end{align}
\begin{align}
    p_{\theta}(d|x) = \frac{\exp(s_{\theta}(d|x)/\tau_1)}{\sum_{d'\in \text{top-}k} \exp(s_{\theta}(d'|x)/\tau_1)} \quad \text{and} \quad q(d|x,y) = \frac{\exp(p_{\text{LM}}(y|d,x)/\tau_2)}{\sum_{d'\in \text{top-}k} \exp(p_{\text{LM}}(y|d',x)/\tau_2)}.
\end{align}
In addition, \citet{basu2024statistical} proposes an objective that also takes expectation over the selected documents. However, our objective models a set’s distribution, and their objective focuses on an item’s distribution. They propose two gradient estimators of their objective, and we compare with the superior one, which minimizes the per-sample loss
\begin{align}
    \sum_{d\in S} [- \log p_{\text{LM}}(y|x,d)-b] \cdot \log p_{\theta}(d|x) \label{eq:statistical_framework}
\end{align}
where $S$ contains $k$ i.i.d. sampled docs following $p_{\theta}(\cdot|x)$.

\paragraph{Hyperparameters.} For all the methods, we tune the learning rate in the grid \{5e-6, 1e-5, 5e-5\} and fix the batch size to 64. For all the methods that involve sampling from the multinomial distribution defined by the softmax of logits, we tune the temperature $\tau\in\{0.02, 0.05, 0.1, 0.5, 1.0, 2.0\}$. For REPLUG, we tune $\tau_1$ in the previously mentioned search range and let $\tau_2=1$. For our method, we tune $\lambda\in$\{1e-3, 1e-4\} in all the experiments. 

\paragraph{Setup.} We train on Wikitext-103, NQ, and TriviaQ all for maximally 3000 steps for all the methods. We experiment with $m\in\{8,16,32\}$ with $k=3$.

\begin{figure}[h!]
    \centering
    \includegraphics[width=0.65\linewidth]{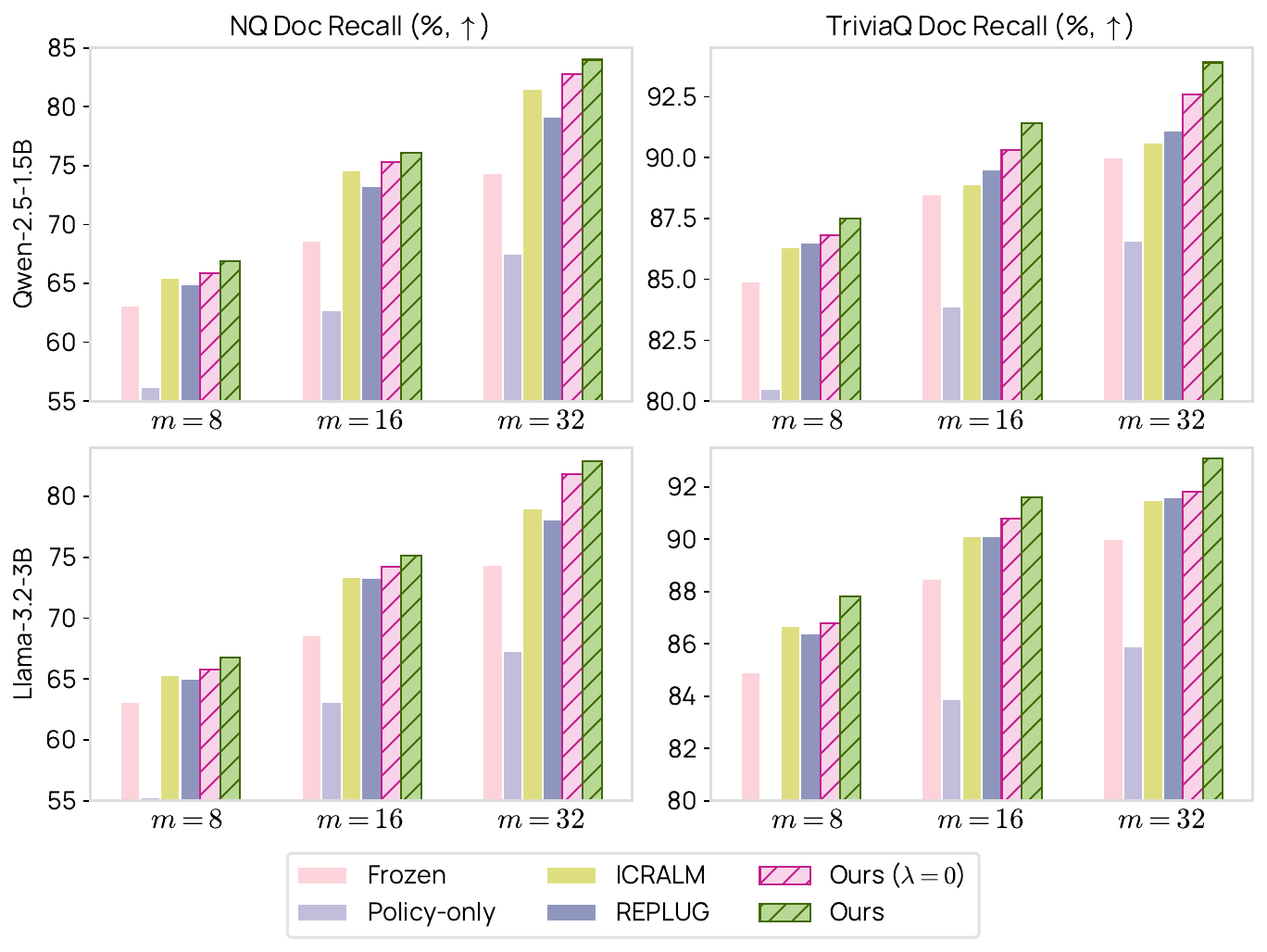}
    \caption{Document recall of rerankers trained via different methods.}
    \label{fig:rag_recall}
\end{figure}

\begin{table}[!h]
\small\centering
\caption{Perplexity on Wikitext-103 ($\downarrow$).} 
\label{tab:rag_wiki}
\begin{tabular}{lcccccc}
\toprule
\multirow{2}{*}{Method} & \multicolumn{3}{c}{Qwen-2.5-1.5B} & \multicolumn{3}{c}{Llama-3.2-3B} \\
 & $m=8$ & $m=16$  & $m=32$ & $m=8$ & $m=16$ & $m=32$   \\
\midrule
Frozen & 9.41 & 10.52 & 11.14 & 7.70 & 8.60 & 9.11 \\
Policy-only & 7.45 & 7.59 & 7.89 & 6.70 & 7.12 & 7.46 \\
ICRALM & 6.61 & 6.99 & 7.64 & 5.37 & 5.56 & 6.09 \\
REPLUG & 6.54 & 6.84 & 7.61 & 5.33 & 5.59 & 5.98 \\
Supervised-only & 6.52 & 6.79 & 7.26 & 5.31 & 5.56 & 5.89 \\
\midrule
Ours & \textbf{6.39} & \textbf{6.68} & \textbf{7.15} & \textbf{5.27} & \textbf{5.47} & \textbf{5.78} \\
\bottomrule
\end{tabular} 
\end{table}

\begin{table}[!h]
\small\centering
\caption{Exact match accuracy on NQ ($\uparrow$).} 
\label{tab:rag_nq_acc}
\begin{tabular}{lcccccc}
\toprule
\multirow{2}{*}{Method} & \multicolumn{3}{c}{Qwen-2.5-1.5B} & \multicolumn{3}{c}{Llama-3.2-3B} \\
 & $m=8$ & $m=16$  & $m=32$ & $m=8$ & $m=16$ & $m=32$   \\
\midrule
Frozen & 23.7 & 26.2 & 26.5 & 28.3 & 30.4 & 31.4 \\
Policy-only & 25.7 & 29.7 & 30.8 & 29.8 & 31.9 & 33.2 \\
ICRALM & 27.5 & 29.7 & 30.8 & 30.8 & 34.1 & 35.0 \\
REPLUG & 26.9 & 29.7 & 31.1 & 31.2 & 33.5 & 34.5 \\
Supervised-only & 27.7 & 31.3 & 32.4 & 31.8 & 34.6 & 35.4 \\
\midrule
Ours & \textbf{28.3} & \textbf{32.2} & \textbf{33.6} & \textbf{32.4} & \textbf{35.1} & \textbf{36.9} \\
\bottomrule
\end{tabular} 
\end{table}


\begin{table}[!h]
\small\centering
\caption{Document recall on NQ ($\uparrow$).} 
\label{tab:rag:nq_recall}
\begin{tabular}{lcccccc}
\toprule
\multirow{2}{*}{Method} & \multicolumn{3}{c}{Qwen-2.5-1.5B} & \multicolumn{3}{c}{Llama-3.2-3B} \\
 & $m=8$ & $m=16$  & $m=32$ & $m=8$ & $m=16$ & $m=32$   \\
\midrule
Frozen & 63.1 & 68.6 & 74.4 & 63.1 & 68.6 & 74.4\\
Policy-only & 56.2 & 62.7 & 67.5 & 55.3 & 63.1 & 67.3\\
ICRALM & 65.5 & 74.6 & 81.5 & 65.3 & 73.4 & 79.0 \\
REPLUG & 64.9 & 73.3 & 79.2 & 65.0 & 73.3 & 78.1 \\
Supervised-only & 65.9 & 75.3 & 82.8 & 65.8 & 74.2 & 81.8 \\
\midrule
Ours & \textbf{66.9} & \textbf{76.1} & \textbf{84.0} & \textbf{66.8} & \textbf{75.1} & \textbf{82.9} \\
\bottomrule
\end{tabular} 
\end{table}

\begin{table}[!h]
\small\centering
\caption{Exact match accuracy on TriviaQ ($\uparrow$).} 
\label{tab:rag_tq_acc}
\begin{tabular}{lcccccc}
\toprule
\multirow{2}{*}{Method} & \multicolumn{3}{c}{Qwen-2.5-1.5B} & \multicolumn{3}{c}{Llama-3.2-3B} \\
 & $m=8$ & $m=16$  & $m=32$ & $m=8$ & $m=16$ & $m=32$   \\
\midrule
Frozen & 55.7 & 58.3 & 59.3 & 71.4 & 72.0 & 73.0 \\
Policy-only & 56.8 & 58.3 & 61.0 & 70.4 & 71.2 & 72.8 \\
ICRALM & 60.3 & 61.1 & 61.9 & 72.1 & 74.5 & 74.6 \\
REPLUG & 60.3 & 61.6 & 63.3 & 72.7 & 74.6 & 75.2 \\
Supervised-only & 61.2 & 62.6 & 65.2 & 73.4 & 74.9 & 75.7\\
\midrule
Ours & \textbf{62.3} & \textbf{63.8} & \textbf{66.2} & \textbf{73.9} & \textbf{75.7} & \textbf{76.8} \\
\bottomrule
\end{tabular} 
\end{table}

\begin{table}[!h]
\small\centering
\caption{Document recall on TriviaQ ($\uparrow$).} 
\label{tab:rag_tq_recall}
\begin{tabular}{lcccccc}
\toprule
\multirow{2}{*}{Method} & \multicolumn{3}{c}{Qwen-2.5-1.5B} & \multicolumn{3}{c}{Llama-3.2-3B} \\
 & $m=8$ & $m=16$  & $m=32$ & $m=8$ & $m=16$ & $m=32$   \\
\midrule
Frozen & 84.9 & 88.5 & 90.0 & 84.9 & 88.5 & 90.0 \\
Policy-only & 80.5 & 83.9 & 86.6 & 80.1 & 83.9 & 85.9 \\
ICRALM & 86.3 & 88.9 & 90.6 & 86.7 & 90.1 & 91.5 \\
REPLUG & 86.5 & 89.5 & 91.1 & 86.4 & 90.1 & 91.6 \\
Supervised-only & 86.8 & 90.3 & 92.6 & 86.8 & 90.8 & 91.8 \\
\midrule
Ours & \textbf{87.5} & \textbf{91.4} & \textbf{93.9} & \textbf{87.8} & \textbf{91.6} & \textbf{93.1} \\
\bottomrule
\end{tabular} 
\end{table}

\begin{figure}[h!]
    \centering
    \includegraphics[width=0.8\linewidth, valign=t]{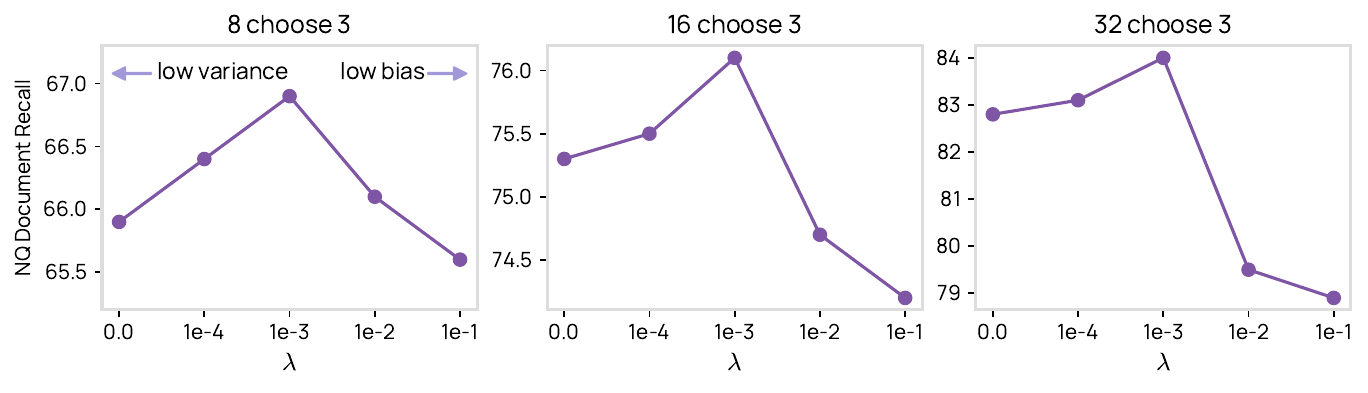}
    \caption{Bias/variance tradeoff of Qwen-2.5-1.5B adjusted by $\lambda$, whose optimal tradeoff is obtained at $\lambda\in(0,1)$.}
    \label{fig:rag_ablation}
\end{figure}

\subsection{Mixture-of-Experts} \label{sec:app:moe}

\paragraph{Hyperparameters.} For all the methods, we tune the learning rate in the grid \{1e-3, 2e-3, 5e-3, 1e-2\} and fix the batch size to 128. For all the methods that involve sampling from the multinomial distribution defined by the softmax of logits, we tune the temperature $\tau\in\{0.02, 0.05, 0.1, 0.5, 1.0, 2.0\}$. For our method, we tune $\lambda\in$\{1e-1, 1e-2, 1e-3, 1e-4\} in all the experiments. For the DefaultMoE baseline, we tune $\beta\in$\{0.6, 0.7, 0.8, 0.9, 0.99\}.

\paragraph{Setup.} For all the methods, we train for 25000 steps. Our model architecture uses n\_embed 64, context window 256, n\_head 8, and $m\in\{8,16,32,64,128\}$ with $k=3$.


\begin{table}[!h]
\setlength{\tabcolsep}{2.5pt}
\small\centering
\caption{Perplexity on Wikitext-103 by MoE ($\downarrow$).} 
\label{tab:moe_wiki}
\begin{tabular}{lccccccccc}
\toprule
\multirow{2}{*}{Method} & \multicolumn{5}{c}{One Layer} & \multicolumn{4}{c}{Two Layers} \\
 & $m=8$ & $m=16$  & $m=32$ & $m=64$ & $m=128$ & $m=8$ & $m=16$  & $m=32$ & $m=64$ \\
\midrule
Deterministic & 86.33 & 84.44 & 86.66 & 86.06 & 85.82 & 71.07 & 69.45 & 68.56 & 69.21 \\
Gaussian & 85.29 & 83.11 & 82.75 & 81.90 & 83.26 & 69.47 & 67.17 & 68.06 & 66.14 \\
DefaultMoE & 87.52 & 89.42 & 87.37 & 86.99 & 84.94 & 70.80 & 70.59 & 69.65 & 67.76 \\
SparseMixer & 87.87 & 86.00 & 86.89 & 86.19 & 86.05 & 71.37 & 68.74 & 68.70 & 69.95 \\
Supervised-only & 84.80 & 80.99 & 77.97 & 74.73 & 74.39 & 69.02 & 64.94 & 62.72 & 59.67 \\
\midrule
Ours & \textbf{84.12} & \textbf{80.09} & \textbf{76.58} & \textbf{73.05} & \textbf{71.84} & \textbf{68.56} & \textbf{64.18} & \textbf{61.58} & \textbf{58.65} \\
\bottomrule
\end{tabular}
\end{table}

\begin{figure}[h!]
    \centering
    \includegraphics[width=\linewidth, valign=t]{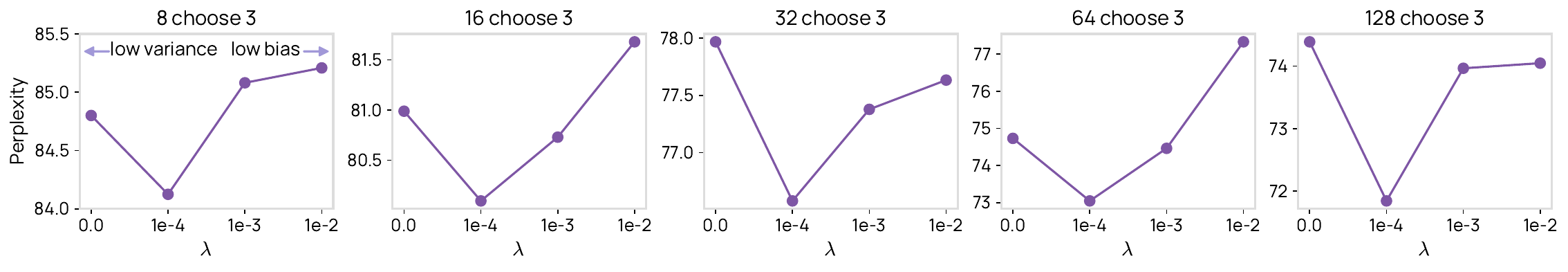}
    \includegraphics[width=\linewidth, valign=t]{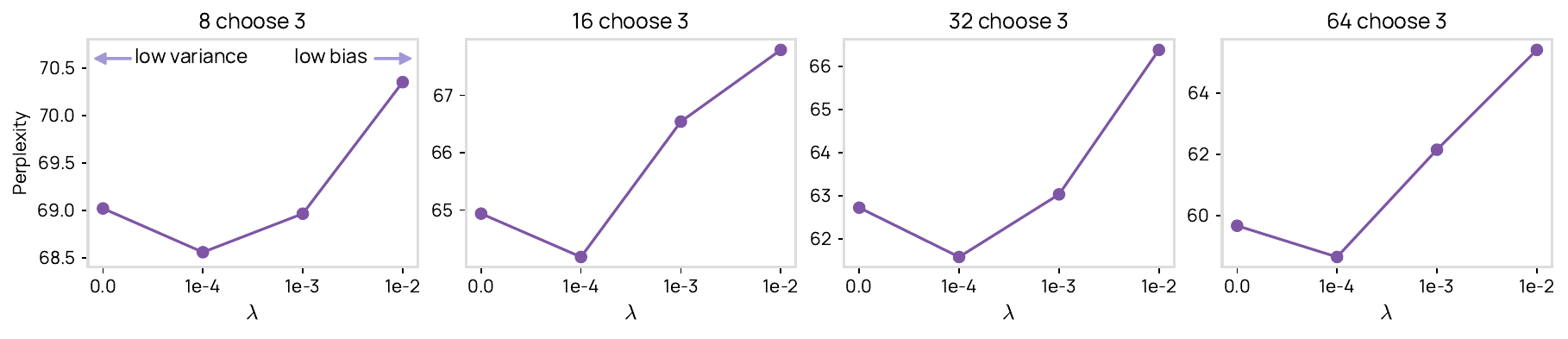}
    \caption{Bias/variance tradeoff adjusted by $\lambda$, showing that the optimal tradeoff is obtained at $\lambda\in(0,1)$ under different $m$'s. The upper row shows the results for 1 MoE layer, and the lower row shows the results for 2 MoE layers.}
    \label{fig:moe_ablation}
\end{figure}

\begin{figure}[h!]
    \centering
    \includegraphics[width=0.27\linewidth, valign=t]{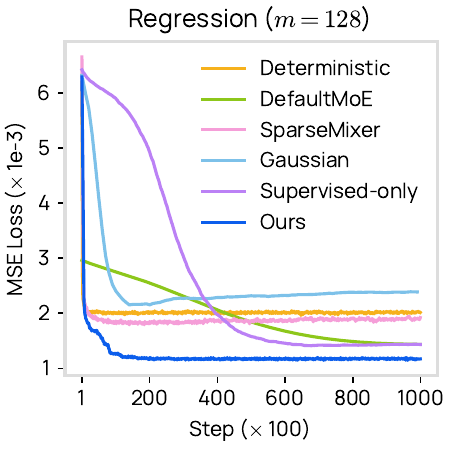}
    \includegraphics[width=0.27\linewidth, valign=t]{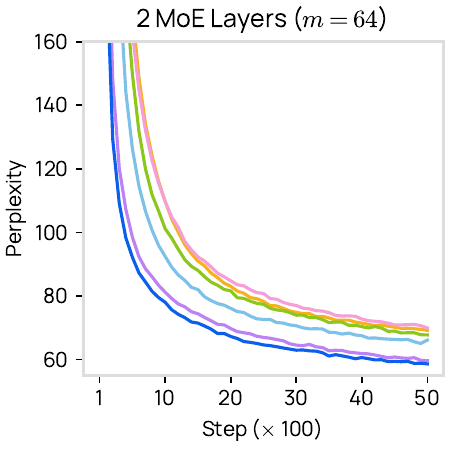}
    \includegraphics[width=0.27\linewidth, valign=t]{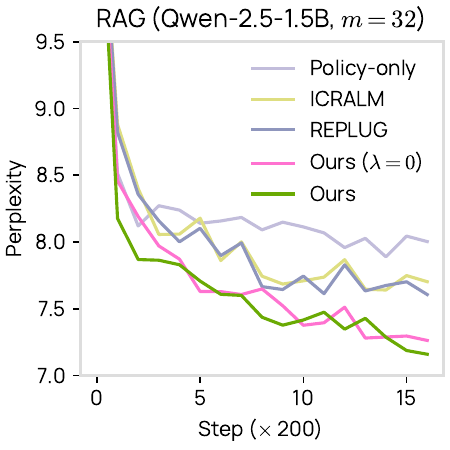}
    \caption{Convergence of difference methods for regression, MoE, and RAG applications.}
    \label{fig:convergence}
\end{figure}

\end{document}